%% file: paper.tex
\documentclass{article}

\newif\ifnips
\nipsfalse

\ifnips
\usepackage[final]{neurips_2019}
\else
\usepackage{a4wide}
\fi

\usepackage{graphicx}

\usepackage[utf8]{inputenc} % allow utf-8 input
\usepackage[T1]{fontenc}    % use 8-bit T1 fonts
\usepackage{hyperref}       % hyperlinks
\usepackage{url}            % simple URL typesetting
\usepackage{booktabs}       % professional-quality tables
\usepackage{amsfonts}       % blackboard math symbols
\usepackage{nicefrac}       % compact symbols for 1/2, etc.
\usepackage{microtype}      % microtypography
\usepackage{bbm}
\usepackage{gensymb}
\usepackage{fixltx2e}

\usepackage{etoolbox, siunitx}
\usepackage{bbm}
\robustify\bfseries
\usepackage{wrapfig}
\usepackage{multirow}
\usepackage{amsmath,amsfonts,amssymb, amsthm}
\usepackage{mathtools}
\usepackage{fixmath}
\usepackage{subcaption}
\usepackage{physics}
\usepackage{multirow, makecell}
\usepackage{hyperref}
\usepackage{tabularx}
\usepackage{float}
\usepackage{enumitem}
\usepackage{boxedminipage}
\usepackage{dashbox}
\usepackage{algorithm2e}
\usepackage{anyfontsize}

\usepackage{cancel}
\usepackage{xcolor,colortbl}

\definecolor{Gray}{gray}{0.95}
\definecolor{DarkGray}{gray}{0.9}

\newcolumntype{a}{>{\columncolor{Gray}[1pt][1pt]}r}
\newcolumntype{b}{>{\columncolor{DarkGray}[1pt][1pt]}r}

\usepackage{footnote}
\makesavenoteenv{tabular}
\makesavenoteenv{table}

\include{preamble}

\title{Adversarial Training and Robustness for \\ Multiple Perturbations}

\author{
 Florian Tram\`er \\
 Stanford University \\
 \and
 \ifnips
 \textbf{Dan Boneh} \\
 \else
 Dan Boneh\\
 \fi
 Stanford University
}

\date{}

\begin{document}

\maketitle
%\vspace{-3em}

\begin{abstract}
Defenses against adversarial examples, such as adversarial training, are typically tailored to a single perturbation type (e.g., small $\ell_\infty$-noise). For other perturbations, these defenses offer no guarantees and, at times, even increase the model's vulnerability.
Our aim is to understand the reasons underlying this robustness trade-off, and to train models that are simultaneously robust to multiple perturbation types.

We prove that a trade-off in robustness to different types of $\ell_p$-bounded and spatial perturbations must exist in a natural and simple statistical setting.
We corroborate our formal analysis by demonstrating similar robustness trade-offs on MNIST and CIFAR10. We propose new multi-perturbation adversarial training schemes, as well as an efficient attack for the $\ell_1$-norm, and use these to show that models trained against multiple attacks fail to achieve robustness competitive with that of models trained on each attack individually. 
In particular, we find that adversarial training with first-order $\ell_\infty, \ell_1$ and $\ell_2$ attacks on MNIST achieves merely $50\%$ robust accuracy, partly because of gradient-masking.
Finally, we propose \emph{affine attacks} that linearly interpolate between perturbation types and further degrade the accuracy of adversarially trained models.
\end{abstract}

\input{intro}
%\input{prelims}
\input{theory}
\input{experiments}

{
%\ifnips
%\newpage
\medskip
\small
%\fi
\bibliographystyle{abbrv}
\bibliography{biblio}
}

\ifdefined\issubmission
\else
\ifnips\newpage\fi
\appendix
\input{apx_experiments}
%\input{apx_l2}
\input{apx_l1attack}
\input{apx_results}
\input{apx_masking}
\input{apx_affine}
\input{apx_proofs}

\fi

\end{document}

%% file: preamble.tex
\usepackage{xcolor}

\newenvironment{proof-sketch}{\noindent{\bf Sketch of Proof}\hspace*{1em}}{\qed\bigskip}
\newenvironment{proof-idea}{\noindent{\bf Proof Idea}\hspace*{1em}}{\qed\bigskip}
\newenvironment{proof-of-lemma}[1]{\noindent{\bf Proof of Lemma #1}\hspace*{1em}}{\qed\bigskip}
\newenvironment{proof-attempt}{\noindent{\bf Proof Attempt}\hspace*{1em}}{\qed\bigskip}

{\makeatletter
 \gdef\xxxmark{%
   \expandafter\ifx\csname @mpargs\endcsname\relax % in minipage?
     \expandafter\ifx\csname @captype\endcsname\relax % in figure/caption?
       \marginpar{\textcolor{red}{xxx~}}% not in a caption or minipage, can use marginpar
     \else
       \textcolor{red}{xxx~}% notice trailing space
     \fi
   \else
     \textcolor{red}{xxx~}% notice trailing space
   \fi}
 \gdef\xxx{\@ifnextchar[\xxx@lab\xxx@nolab}
 \long\gdef\xxx@lab[#1]#2{{\bf [\xxxmark \textcolor{red}{#2} ---{\sc #1}]}}
 \long\gdef\xxx@nolab#1{{\bf [\xxxmark \textcolor{red}{#1}]}}
}

\newcommand{\R}{\mathbb{R}}

\DeclareMathOperator*{\Exp}{\mathbb{E}}
\DeclareMathOperator*{\argmax}{arg\,max}
\DeclareMathOperator*{\argmin}{arg\,min}

\newcommand{\calA}{\ensuremath{\mathcal{A}}}

\newcommand{\calD}{\ensuremath{\mathcal{D}}}

\newcommand{\calG}{\ensuremath{\mathcal{G}}}
\newcommand{\calH}{\ensuremath{\mathcal{H}}}

\newcommand{\calN}{\ensuremath{\mathcal{N}}}

\newcommand{\calP}{\ensuremath{\mathcal{P}}}
\newcommand{\calQ}{\ensuremath{\mathcal{Q}}}
\newcommand{\calR}{\ensuremath{\mathcal{R}}}

\newcommand{\calZ}{\ensuremath{\mathcal{Z}}}

\theoremstyle{plain}
\newtheorem{theorem}{Theorem}
\newtheorem{lemma}[theorem]{Lemma}
\newtheorem{claim}[theorem]{Claim}

\theoremstyle{definition}

\newtheorem{myremark}[theorem]{Remark}

\newenvironment{myenumerate}
{\begin{enumerate}[leftmargin=5.5mm, itemsep=1pt, topsep=0pt]}
	{\end{enumerate}}

\renewcommand{\wp}{\ \textnormal{w.p.}\ }

\newcommand{\sign}{\textrm{sign}}

\newcommand{\relu}{\textrm{ReLU}}

\renewcommand{\vec}[1]{\mathbold{#1}}

\newcommand{\Var}{\operatorname{Var}}
\newcommand{\OPT}{\operatorname{OPT}}

%% file: intro.tex
\section{Introduction}
Adversarial examples~\cite{szegedy2013intriguing, goodfellow2014explaining} are proving to be an inherent blind-spot in machine learning (ML) models. Adversarial examples highlight the tendency of ML models to learn superficial and brittle data statistics~\cite{jo2017measuring, geirhos2018imagenet,ilyas2019adversarial}, and present a security risk for models deployed in cyber-physical systems (e.g., virtual assistants~\cite{carlini2016hidden}, malware detectors~\cite{grosse2016adversarial} or ad-blockers~\cite{tramer2018adblock}).

Known successful defenses are tailored to a specific perturbation type (e.g., a small $\ell_p$-ball~\cite{madry2018towards,raghunathan2018certified,wong2018provable} or small spatial transforms~\cite{engstrom2017rotation}). These defenses provide empirical (or certifiable) robustness guarantees for one perturbation type, but typically offer no guarantees against other attacks~\cite{sharma2017attacking, schott2018towards}. Worse, increasing robustness to one perturbation type has sometimes been found to increase vulnerability to others~\cite{engstrom2017rotation, schott2018towards}. This leads us to the central problem considered in this paper:
\begingroup
\addtolength\leftmargini{-0.2in}
%\ifnips\vspace{-0.5em}\fi
\begin{quote}
\emph{Can we achieve adversarial robustness to different types of perturbations simultaneously?}
\end{quote} 
%\ifnips\vspace{-0.5em}\fi
\endgroup
Note that even though prior work has attained robustness to different perturbation types~\cite{madry2018towards, schott2018towards, engstrom2017rotation}, these results may not compose. For instance, an ensemble of two classifiers---each of which is robust to a single type of perturbation---may be robust to neither perturbation. Our aim is to study the extent to which it is possible to learn models that are \emph{simultaneously} robust to multiple types of perturbation.

To gain intuition about this problem, 
we first study a simple and natural classification task, that has been used to analyze trade-offs between standard and adversarial accuracy~\cite{tsipras2019robustness}, and the sample-complexity of adversarial generalization~\cite{schmidt2018adversarially}. 
We define \emph{Mutually Exclusive Perturbations (MEPs)} as pairs of perturbation types for which robustness to one type implies vulnerability to the other. For this task, we prove that $\ell_\infty$ and $\ell_1$-perturbations are MEPs and that $\ell_\infty$-perturbations and input rotations and translations~\cite{engstrom2017rotation} are also MEPs.
Moreover, for these MEP pairs, we find that robustness to either perturbation type requires fundamentally different features. The existence of such a trade-off for this simple classification task suggests that it may be prevalent in more complex statistical settings.
%, e.g., in the superficial and brittle data statistics that ML models learn in real datasets.

To complement our formal analysis, we introduce new adversarial training schemes for multiple perturbations. For each training point, these schemes build adversarial examples for all perturbation types and then train either on all examples (the ``avg'' strategy) or only the worst example (the ``max'' strategy). These two strategies respectively minimize the \emph{average} error rate across perturbation types, or the error rate against an adversary that picks the worst perturbation type for each input.

For adversarial training to be practical, we also need efficient and strong attacks~\cite{madry2018towards}. We show that Projected Gradient Descent~\cite{kurakin2016scale, madry2018towards} is inefficient in the $\ell_1$-case, and design a new attack, \emph{Sparse $\ell_1$ Descent} (SLIDE), that is both efficient and competitive with strong optimization attacks~\cite{chen2018ead}, 

%We also design a new Projected Gradient Descent (PGD) attack for the $\ell_1$-norm, of independent interest. 
%We show that iterated steepest descent~\cite{madry2018tutorial} leads to a poor $\ell_1$-attack, and propose an improved gradient update step to rival the state-of-the-art (and much more expensive) EAD attack~\cite{chen2018ead}.

%As we will show, achieving robustness to all $\ell_p$-bounded perturbations requires novel techniques, to circumvent pernicious \emph{gradient-masking}~\cite{papernot2016practical, tramer2018ensemble,athalye2018obfuscated} issues.

\begin{figure}[t]
	\small
	\centering
	\begin{subfigure}[t]{0.49\textwidth}
		\centering
		\includegraphics[width=\textwidth]{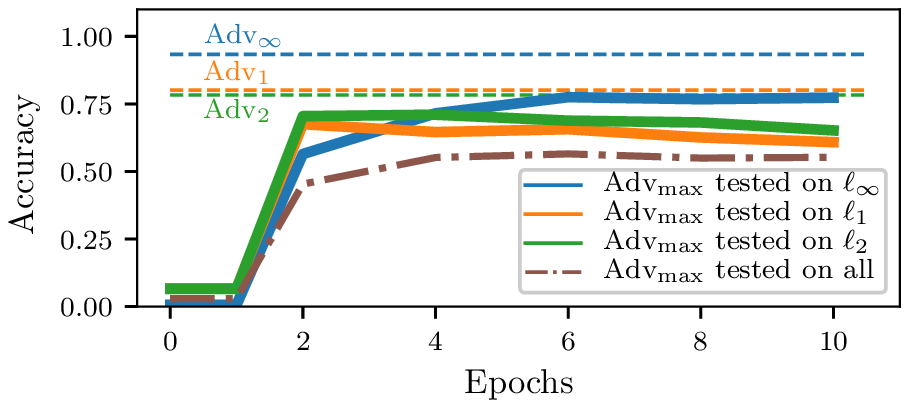}
		%\ifnips\vspace*{-5mm}\fi
		\caption{MNIST models trained on $\ell_1,\ell_2$ \& $\ell_\infty$ attacks.}
	\end{subfigure}
	\hfill
	\begin{subfigure}[t]{0.49\textwidth}
		\centering
		\includegraphics[width=\textwidth]{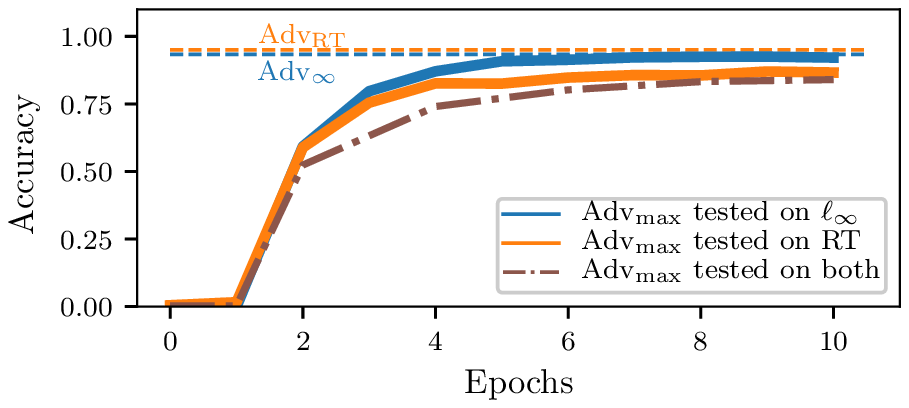}
		%\ifnips\vspace*{-5mm}\fi
		\caption{MNIST models trained on $\ell_\infty$ and RT attacks.}
	\end{subfigure}
	%\ifnips\\[0.25em]\else\\[0.75em]\fi
	\\[0.75em]
	\begin{subfigure}{0.49\textwidth}
		\centering
		\includegraphics[width=\textwidth]{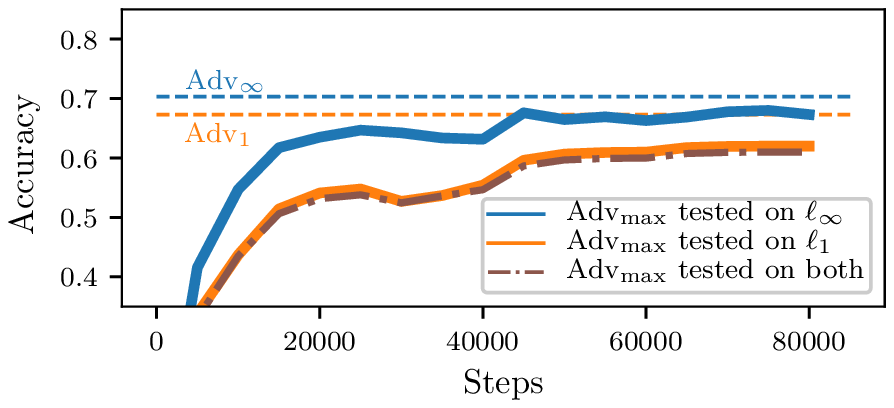}
		%\ifnips\vspace*{-5mm}\fi
		\caption{CIFAR10 models trained on $\ell_1$ and $\ell_\infty$ attacks.}
	\end{subfigure}
	\hfill
	\begin{subfigure}{0.49\textwidth}
		\centering
		\includegraphics[width=\textwidth]{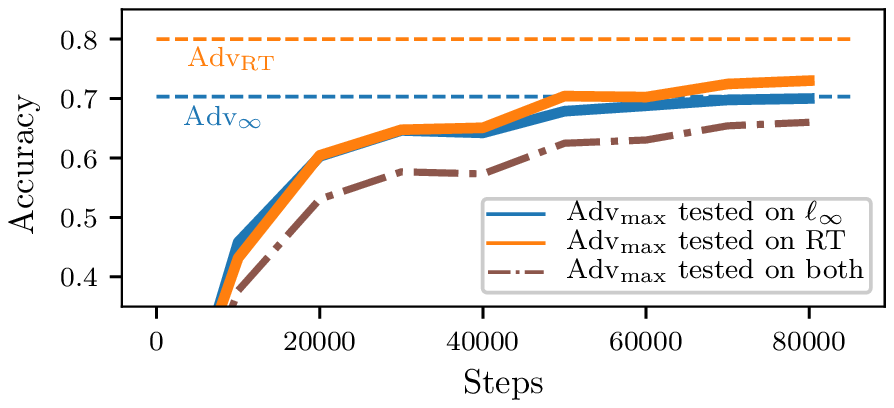}
		%\ifnips\vspace*{-5mm}\fi
		\caption{CIFAR10 models trained on $\ell_\infty$ and RT attacks.}
	\end{subfigure}
	%\ifnips\vspace{-0.5em}\fi
	\caption{\textbf{Robustness trade-off on MNIST (top) and CIFAR10 (bottom).} For a union of $\ell_p$-balls (left), or of $\ell_\infty$-noise and rotation-translations (RT) (right), we train models Adv$_\text{max}$ on the strongest perturbation-type for each input. We report the test accuracy of Adv$_\text{max}$
	against each individual perturbation type (solid line) and against their union (dotted brown line). The vertical lines show the adversarial accuracy of models trained and evaluated on a single perturbation type.
	%\ifnips\\[-2em]\fi
    }
	\label{fig:training}
\end{figure}

We experiment with MNIST and CIFAR10. MNIST is an interesting case-study, as \emph{distinct} models from prior work attain strong robustness to all perturbations we consider~\cite{madry2018towards,schott2018towards, engstrom2017rotation}, yet no \emph{single} classifier is robust to all attacks~\cite{schott2018towards, openreview,engstrom2017rotation}.
For models trained on multiple $\ell_p$-attacks ($\ell_1, \ell_2, \ell_\infty$ for MNIST, and $\ell_1, \ell_\infty$ for CIFAR10), or on both $\ell_\infty$ and spatial transforms~\cite{engstrom2017rotation}, we confirm a noticeable robustness trade-off. 
Figure~\ref{fig:training} plots the test accuracy of models Adv$_\text{max}$ trained using our ``max'' strategy. In all cases, robustness to multiple perturbations comes at a cost---usually of $5$-$10\%$ additional error---compared to models trained against each attack individually (the horizontal lines). 
%The robustness trade-off is often due to poor \emph{adversarial generalization}~\cite{schmidt2018adversarially}, i.e., the  Adv$_\text{max}$ models get $100\%$ \emph{training} accuracy for all perturbations they are trained against.

Robustness to $\ell_1, \ell_2$ and $\ell_\infty$-noise on MNIST is a striking failure case, where the robustness trade-off is compounded by  \emph{gradient-masking}~\cite{papernot2016practical, tramer2018ensemble, athalye2018obfuscated}. Extending prior observations~\cite{madry2018towards,schott2018towards,li2018second}, we show that models trained against an $\ell_\infty$-adversary learn representations that \emph{mask gradients} for attacks in other $\ell_p$-norms. When trained against first-order $\ell_1, \ell_2$ and $\ell_\infty$-attacks, the model learns to resist $\ell_\infty$-attacks while giving the illusion of robustness to $\ell_1$ and  $\ell_2$ attacks. This model only achieves $52\%$ accuracy when evaluated on gradient-free attacks~\cite{brendel2018decision,schott2018towards}.
This shows that, unlike previously thought~\cite{tsipras2019robustness}, adversarial training with strong first-order attacks can suffer from gradient-masking. We thus argue that attaining robustness to $\ell_p$-noise on MNIST requires new techniques (e.g., training on expensive gradient-free attacks, or scaling certified defenses to multiple perturbations). 

MNIST has sometimes been said to be a poor dataset for evaluating adversarial examples defenses, as some attacks are easy to defend against (e.g., input-thresholding or binarization works well for $\ell_\infty$-attacks~\cite{tsipras2019robustness,schott2018towards}). Our results paint a more nuanced view: the simplicity of these $\ell_\infty$-defenses becomes a disadvantage when training against multiple $\ell_p$-norms. We thus believe that MNIST should not be abandoned as a benchmark just yet. Our inability to achieve multi-$\ell_p$ robustness for this simple dataset raises questions about the viability of scaling current defenses to more complex tasks.

Looking beyond adversaries that choose from a union of perturbation types, 
we introduce a new \emph{affine adversary} that may linearly interpolate between perturbations (e.g., by compounding $\ell_\infty$-noise with a small rotation). We prove that for locally-linear models, robustness to a union of $\ell_p$-perturbations implies robustness to affine attacks. In contrast, affine combinations of $\ell_\infty$ and spatial perturbations are provably stronger than either perturbation individually. We show that this discrepancy translates to neural networks trained on real data. Thus, in some cases, attaining robustness to a union of perturbation types remains insufficient against a more creative adversary that composes perturbations.

Our results show that despite recent successes in achieving robustness to single perturbation types, many obstacles remain towards attaining truly robust models. Beyond the robustness trade-off, efficient computational scaling of current defenses to multiple perturbations remains an open problem.

The code used for all of our experiments can be found here: \url{https://github.com/ftramer/MultiRobustness}

\ifnips
Proofs of all theorems, experimental setups, and additional experiments are in the full version of  this extended abstract~\cite{fullversion}.
\fi

%Beyond the robustness trade-off we uncover, the scalability of adversarial training to multiple perturbation types is another point of concern (e.g., our new adversarial training modes scale linearly with the number of perturbation types). Even if robustness to a large set of perturbations types were attainable, this would require new and substantially faster training techniques.

%% file: theory.tex
\section{Theoretical Limits to Multi-perturbation Robustness}
\label{sec:theory}

We study statistical properties of adversarial robustness in a natural statistical model introduced in~\cite{tsipras2019robustness}, and which exhibits many phenomena observed on real data, such as trade-offs between robustness and accuracy~\cite{tsipras2019robustness} or a higher sample complexity for robust generalization~\cite{schott2018towards}.
This model also proves useful in analyzing and understanding adversarial robustness for multiple perturbations. Indeed, we prove a number of results that correspond to phenomena we observe on real data, in particular trade-offs in robustness to different $\ell_p$ or rotation-translation attacks~\cite{engstrom2017rotation}. %

We follow a line of works that study distributions for which adversarial examples exist \emph{unconditionally}~\cite{tsipras2019robustness, khoury2019geometry, shafahi2019adversarial, fawzi2018adversarial, gilmer2018adversarial, mahloujifar2018curse}. These distributions, including ours, are much simpler than real-world data, and thus need not be evidence that adversarial examples are inevitable in practice. Rather, we hypothesize that current ML models are highly vulnerable to adversarial examples because they learn superficial data statistics~\cite{jo2017measuring, geirhos2018imagenet,ilyas2019adversarial} that share some properties of these simple distributions.

In prior work, a robustness trade-off for $\ell_\infty$ and $\ell_2$-noise is shown in~\cite{khoury2019geometry} for data distributed over two concentric spheres. Our conceptually simpler model has the advantage of yielding results beyond $\ell_p$-norms (e.g., for spatial attacks) and which apply symmetrically to both classes. Building on work by Xu et al.~\cite{xu2009robustness}, Demontis et al.~\cite{demontis2016security} show a robustness trade-off for dual norms (e.g., $\ell_\infty$ and $\ell_1$-noise) in linear classifiers.

\subsection{Adversarial Risk for Multiple Perturbation Models}
\label{ssec:risk}

Consider a classification task for a distribution $\calD$ over examples $\vec{x} \in \R^d$ and labels $y \in [C]$. Let $f : \R^d \to [C]$ denote a classifier and let $l(f(\vec{x}), y)$ be the zero-one loss (i.e., $\mathbbm{1}_{f(\vec{x}) \neq y}$).

We assume $n$ \emph{perturbation types}, each characterized by a set $S$ of allowed perturbations for an input $\vec{x}$. The set $S$ can be an $\ell_p$-ball~\cite{szegedy2013intriguing, goodfellow2014explaining} or capture other perceptually small transforms such as image rotations and translations~\cite{engstrom2017rotation}. For a perturbation $\vec{r} \in S$, an adversarial example is $\hat{\vec{x}} = \vec{x} + \vec{r}$ (this is pixel-wise addition for $\ell_p$ perturbations, but can be a more complex operation, e.g., for rotations).

For a perturbation set $S$ and model $f$, we define $\calR_{\text{adv}}(f; S) \coloneqq \Exp_{(\vec{x}, y) \sim \calD} \left[\max_{\vec{r} \in S}\ l(f(\vec{x}+\vec{r}), y)\right] $ as the adversarial error rate.
To extend $\calR_{\text{adv}}$ to multiple perturbation sets $S_1, \dots, S_n$, we can consider the \emph{average} error rate for each $S_i$, denoted $\calR^{\text{avg}}_{\text{adv}}$. This metric most clearly captures the trade-off in robustness across independent perturbation types, but is not the most appropriate from a security perspective on adversarial examples. A more natural metric, denoted  $\calR^{\text{max}}_{\text{adv}}$, is the error rate against an adversary that picks, for each input, the worst perturbation from the \emph{union} of the $S_i$. More formally,
\begin{equation}
\label{eq:robustness_measures}
\textstyle \calR^{\text{max}}_{\text{adv}}(f; S_1, \dots, S_n) \coloneqq \calR_{\text{adv}}(f; \cup_i S_i) \;, \quad \calR^{\text{avg}}_{\text{adv}}(f; S_1, \dots, S_n) \coloneqq \frac1n \sum_i \calR_{\text{adv}}(f; S_i) \;.
\end{equation}

Most results in this section are \emph{lower bounds} on $\calR^{\text{avg}}_{\text{adv}}$, which also hold for $\calR^{\text{avg}}_{\text{max}}$ since $\calR^{\text{max}}_{\text{adv}} \geq \calR^{\text{avg}}_{\text{adv}}$.

Two perturbation types $S_1, S_2$ are \emph{Mutually Exclusive Perturbations (MEPs)}, if $\calR^{\text{avg}}_{\text{adv}}(f; S_1, S_2) \geq {1}/{|C|}$ for all models $f$ (i.e., no model has non-trivial average risk against both perturbations).

\subsection{A binary classification task} 
\label{ssec:dist}
We analyze the adversarial robustness trade-off for different perturbation types in a natural statistical model introduced by Tsipras et al.~\cite{tsipras2019robustness}.
Their binary classification task consists of input-label pairs $(\vec{x}, y)$ sampled from a distribution $\calD$ as follows (note that $\calD$ is $(d+1)$-dimensional):
\begin{equation}
\textstyle
y \stackrel{\emph{u.a.r}}{\sim} \{-1, +1\}, \quad 
x_0 = 
\begin{cases} 
+y, \wp p_0, \\
-y, \wp 1-p_0
\end{cases}, \quad
x_1, \dots, x_d \stackrel{\emph{i.i.d}}{\sim} \calN(y\eta, 1) \;,
\end{equation}
where $p_0\geq 0.5$, $\calN(\mu, \sigma^2)$ is the normal distribution and $\eta = {\alpha}/{\sqrt{d}}$ for some positive constant $\alpha$. 

For this distribution, Tsipras et al.~\cite{tsipras2019robustness} show a trade-off between standard and adversarial accuracy (for $\ell_\infty$ attacks), by drawing a distinction between the ``robust'' feature $x_0$ that small $\ell_\infty$-noise cannot manipulate, and the ``non-robust'' features $x_1, \dots, x_d$ that can be fully overridden by small $\ell_\infty$-noise.

\subsection{Small \texorpdfstring{$\ell_\infty$}{Linf} and \texorpdfstring{$\ell_1$}{L1} Perturbations are Mutually Exclusive}
\label{ssec:tradeoff_lp}

The starting point of our analysis is the observation that the robustness of a feature depends on the considered perturbation type. To illustrate, we recall two classifiers from~\cite{tsipras2019robustness} that operate on disjoint feature sets. The first, $f(\vec{x}) = \sign(x_0)$, achieves accuracy $p_0$ for all $\ell_\infty$-perturbations with $\epsilon < 1$ but is highly vulnerable to $\ell_1$-perturbations of size $\epsilon \geq 1$. The second classifier, $h(\vec{x}) = \sign(\sum_{i=1}^d x_i)$ is robust to $\ell_1$-perturbations of average norm below $\Exp[\sum_{i=1}^d x_i] =\Theta(\sqrt{d})$, yet it is fully subverted by a $\ell_\infty$-perturbation that shifts the features $x_1, \dots, x_d$ by $\pm2\eta = \Theta(1/\sqrt{d})$.
%
%For instance, $x_0$ is a strong feature for $\ell_\infty$-perturbations, but a very weak feature for $\ell_1$-perturbations (an input's average $\ell_1$-norm, $\Exp[\norm{\vec{x}}_1] = O(d)$, is much larger than $|x_0|$). Conversely, the features $x_1, \dots, x_d$ are robust to small $\ell_1$-perturbations, but not to $\ell_\infty$-perturbations. %Thus, this data model exhibits a tension between robustness to different $\ell_p$-perturbations.
%\paragraph{Robustness to one perturbation type is easy.}
%
%\paragraph{Robustness to $\ell_1$ and $\ell_\infty$-perturbations is impossible.}
%
%The above classifiers $f$ and $h$ are robust to either $\ell_\infty$-bounded or $\ell_1$-bounded perturbations, but not against both. An $\ell_1$ perturbation with $\epsilon > 1$ suffices to reduce the accuracy of $f$ to $0\%$.
%Conversely,  a weak $\ell_\infty$-perturbation that shifts the features $x_1, \dots, x_2$ by $\pm2\eta$ fully subverts $h$.
We prove that this tension between $\ell_\infty$ and $\ell_1$ robustness, and of the choice of ``robust'' features, is inherent for this task:
\begin{theorem}%[Robustness trade-off between $\ell_\infty$ and $\ell_1$-norms]
	\label{thm:linf_l1}
	Let $f$ be a classifier for $\calD$. Let $S_{\infty}$ be the set of $\ell_\infty$-bounded perturbations with $\epsilon = 2\eta$, and $S_{1}$ the set of $\ell_1$-bounded perturbations with $\epsilon=2$. Then, 
	$
	\calR_{\text{adv}}^{\text{avg}}(f; S_{\infty}, S_{1})  \geq 1/2 \;.
	$
\end{theorem}
\vspace{-0.25em}
The proof is in Appendix~\ref{apx:proof_l1_linf}. The bound shows that no classifier can attain better $\calR_{\text{adv}}^{\text{avg}}$ (and thus $\calR_{\text{adv}}^{\text{max}}$) than a trivial constant classifier $f(x)=1$, which satisfies $\calR_{\text{adv}}(f; S_{\infty}) = \calR_{\text{adv}}(f; S_{1})  = 1/2$.

Similar to~\cite{demontis2016security}, our analysis extends to arbitrary dual norms $\ell_p$ and $\ell_q$ with $1/p + 1/q = 1$ and $p < 2$. The perturbation required to flip the features $x_1, \dots, x_n$ has an $\ell_p$ norm of $\Theta(d^{\frac1p - \frac12}) = \omega(1)$ and an $\ell_q$ norm of $\Theta(d^{\frac1q - \frac12}) = \Theta(d^{\frac12 - \frac1p}) =o(1)$. Thus, feature $x_0$ is more robust than features $x_1, \dots, x_n$ with respect to the $\ell_q$-norm, whereas for the dual $\ell_p$-norm the situation is reversed.

%\paragraph{Robustness to $\ell_2$  perturbations.}
%The above data distribution exhibits a stark divide between robustness to $\ell_\infty$ and $\ell_1$ perturbations, where strong features in one attack model correspond to weak features in the other. What about an $\ell_2$ adversary? With a slight generalization of the distribution $\calD$, we show in Appendix~\ref{apx:l2} that $\ell_2$ adversaries can effectively interpolate between the two types of features, while maintaining the tension between $\ell_1$ and $\ell_\infty$ adversaries.

%\paragraph{Empirical results with feature selection.}
%\todo{do we really need this?}

\subsection{Small \texorpdfstring{$\ell_\infty$}{Linf} and Spatial Perturbations are (nearly) Mutually Exclusive}
\label{ssec:tradeoff_RT}

We now analyze two other orthogonal perturbation types, $\ell_\infty$-noise and rotation-translations~\cite{engstrom2017rotation}. In some cases, increasing robustness to $\ell_\infty$-noise has been shown to decrease robustness to rotation-translations~\cite{engstrom2017rotation}.
We prove that such a trade-off is inherent for our binary classification task.

To reason about rotation-translations, we assume that the features $x_i$ form a $2$D grid. We also let $x_0$ be distributed as $\calN(y, \alpha^{-2})$, a technicality that does not qualitatively change our prior results.
Note that the distribution of the features $x_1, \dots, x_d$ is permutation-invariant. Thus, the only power of a rotation-translation adversary is to ``move'' feature $x_0$. Without loss of generality, we identify a small rotation-translation of an input $\vec{x}$ with a permutation of its features that sends $x_0$ to one of $N$ fixed positions (e.g., with translations of $\pm 3$px as in~\cite{engstrom2017rotation}, $x_0$ can be moved to $N=49$ different positions). 

%Without loss of generality, we identify rotations or translations of an input $\vec{x}$ with a fixed set of permutations of its feature indices. \footnote{Assuming that no aliasing occurs, an image rotation is a (structured) permutation of its pixels. For translations, we can assume for simplicity that $x_0$ is not near an image border, and that features ``cropped out'' by the translation are replaced by identically distributed features. A translation is then also a feature permutation.} More formally, our rotation-translation adversary picks a permutation from some fixed set of permutations over $[0, d]$. We constrain the strength of this adversary by limiting the number of positions, denoted $N$, that feature $x_0$ can be sent to (e.g., with translations of $\pm 3$ pixels as in~\cite{engstrom2017rotation}, each feature can be moved to one of $N=49$ different positions, including its original position). 

A model can be robust to these permutations by ignoring the $N$ positions that feature $x_0$ can be moved to, and focusing on the remaining permutation-invariant features. Yet, this model is vulnerable to $\ell_\infty$-noise, as it ignores $x_0$.  In turn, a model that relies on feature $x_0$ can be robust to $\ell_\infty$-perturbations, but is vulnerable to a spatial perturbation that ``hides'' $x_0$ among other features. Formally, we show:
\begin{theorem}%[Robustness trade-off between $\ell_\infty$ and spatial perturbations]
\label{thm:linf_RT}
	Let $f$ be a classifier for $\calD$ (with $x_0 \sim \calN(y, \alpha^{-2})$). Let $S_{\infty}$ be the set of $\ell_\infty$-bounded perturbations with $\epsilon = 2\eta$, and $S_{\text{RT}}$ be the set of perturbations for an RT adversary with budget $N$. Then, 
	$
	\calR_{\text{adv}}^{\text{avg}}(f; S_{\infty}, S_{\text{RT}})  \geq 1/2 - O({1}/{\sqrt{N}}) \;.
	$
\end{theorem}
The proof, given in Appendix~\ref{apx:proof_RT}, is non-trivial and yields an asymptotic lower-bound on $\calR_{\text{adv}}^{\text{avg}}$. We can also provide tight numerical estimates for concrete parameter settings (see Appendix~\ref{apx:remark_RT}).

\subsection{Affine Combinations of Perturbations}
\label{ssec:affine}

We defined $\calR_{\text{adv}}^{\text{max}}$ as the error rate against an adversary that may choose a different perturbation type for each input. If a model were robust to this adversary, what can we say about the robustness to a more creative adversary that \emph{combines} different perturbation types?
To answer this question, we introduce a new adversary that mixes different attacks by linearly interpolating between perturbations. %We show that this adversary is sometimes stronger than one that is limited to choosing between individual perturbation types.

For a perturbation set $S$ and $\beta \in [0,1]$, we denote $\beta\cdot S$ the set of perturbations scaled down by $\beta$. For an $\ell_p$-ball with radius $\epsilon$, this is the ball with radius $\beta\cdot\epsilon$. For rotation-translations, the attack budget $N$ is scaled to $\beta\cdot N$.
For two sets $S_1,S_2$, we define $S_\text{affine}(S_1, S_2)$ as the set of perturbations that compound a perturbation $\vec{r}_1 \in \beta \cdot S_1$ with a perturbation $\vec{r}_2 \in (1-\beta) \cdot S_2$, for any $\beta \in [0, 1]$.

Consider one adversary that chooses, for each input, $\ell_p$ or $\ell_q$-noise from balls $S_p$ and $S_q$, for $p, q > 0$. The affine adversary picks perturbations from the set $S_\text{affine}$ defined as above. We show:
\begin{claim}%[Affine combinations of $\ell_p$-perturbations do not affect linear models]
	\label{thm:linear_affine}
	For a linear classifier $f(\vec{x}) = \sign(\vec{w}^T \vec{x} + b)$, we have $\calR_{\text{adv}}^{\text{max}}(f; S_p, S_q) = \calR_{\text{adv}}(f; S_\text{affine})$.
\end{claim}
Thus, for linear classifiers, robustness to a union of $\ell_p$-perturbations implies robustness to affine adversaries (this holds for any distribution). The proof, in Appendix~\ref{apx:proof_affine_linear} extends to models that are \emph{locally linear} within balls $S_p$ and $S_q$ around the data points. {For the distribution $\calD$ of Section~\ref{ssec:dist}, we can further show that there are settings (distinct from the one in Theorem~\ref{thm:linf_l1}) where: (1) robustness against a union of $\ell_\infty$ and $\ell_1$-perturbations is possible; (2) this requires the model to be non-linear; (3) yet, robustness to affine adversaries is impossible (see Appendix~\ref{apx:affine_LP_nonlinear} for details).} Our experiments in Section~\ref{sec:experiments} show that neural networks trained on CIFAR10 have a behavior that is consistent with locally-linear models, in that they are as robust to affine adversaries as against a union of $\ell_p$-attacks.

%For our distribution $\calD$, we can actually show that there are settings where non-linearity is necessary to get robustness against $S_\text{U}$, but that the resulting models remain vulnerable to affine combinations of perturbations (the $\ell_1$ and $\ell_\infty$-bounds in the below theorem are incomparable to the ones in Theorem~\ref{thm:linf_l1}, so the two results are not contradictory):
%

In contrast, compounding $\ell_\infty$ and spatial perturbations yields a stronger attack, even for linear models:

\begin{theorem}%[Affine combinations of $\ell_\infty$ and spatial perturbations can affect linear models]
	\label{thm:affine_RT}
Let $f(\vec{x}) = \sign(\vec{w}^T \vec{x} + b)$ be a linear classifier for $\calD$ (with $x_0 \sim \calN(y, \alpha^{-2})$). Let $S_\infty$ be some $\ell_\infty$-ball and $S_{\text{RT}}$ be rotation-translations with budget $N>2$. Define $S_{\text{affine}}$ as above. Assume $w_0 > w_i > 0, \forall i \in[1, d]$. Then $\calR_{\text{adv}}(f; S_{\text{affine}}) >  \calR_{\text{adv}}^{\text{max}}(f; S_\infty, S_{\text{RT}})$. 
\end{theorem}

This result (the proof is in Appendix~\ref{apx:proof_affine_RT}) draws a distinction between the strength of affine combinations of $\ell_p$-noise, and combinations of $\ell_\infty$ and spatial perturbations. It also shows that robustness to a union of perturbations can be insufficient against a more creative affine adversary. 
These results are consistent with behavior we observe in models trained on real data (see Section~\ref{sec:experiments}).

%% file: experiments.tex
\section{New Attacks and Adversarial Training Schemes}
\label{sec:adv-training}

We complement our theoretical results with empirical evaluations of the robustness trade-off on MNIST and CIFAR10. To this end, we first introduce new adversarial training schemes tailored to the multi-perturbation risks defined in Equation~\eqref{eq:robustness_measures}, as well as a novel attack for the $\ell_1$-norm.

\paragraph{Multi-perturbation adversarial training.}
%\label{sec:adv-training}
Let
\[\hat{\calR}_{\text{adv}}(f; S)=\sum_{i=1}^m \max_{\vec{r} \in S} L(f(\vec{x}^{(i)} + \vec{r}), y^{(i)})\;,\]
bet the empirical adversarial risk, where $L$ is the training loss and $D$ is the training set.
For a single perturbation type, $\hat{\calR}_{\text{adv}}$ can be minimized with \emph{adversarial training}~\cite{madry2018towards}: the maximal loss is approximated by an attack procedure $\calA(\vec{x})$, such that $\max_{\vec{r} \in S} L(f(\vec{x}+ \vec{r}), y) \approx L(f(\calA(\vec{x})), y)$.

For $i \in [1,d]$, let $\calA_i$ be an attack for the perturbation set $S_i$. The two multi-attack robustness metrics introduced in Equation~\eqref{eq:robustness_measures}  immediately yield the following natural adversarial training strategies:

\begin{myenumerate}
	\item \textbf{``Max'' strategy: } For each input $\vec{x}$, we train on the strongest adversarial example from all attacks, i.e., the $\max$ in $\hat{\calR}_{\text{adv}}$ is replaced by $L(f(\calA_{k^*}(\vec{x})), y)$, for $k^* = \argmax_{k} L(f(\calA_k(\vec{x})), y)$.
	
	%This approach has two drawbacks. First, it is wasteful, as only one of the computed adversarial examples is used. Second, it can overfit to attacks with low \emph{adversarial generalization gap}~\cite{}: suppose a model has $60\%$ train and test accuracy under attack $\calA_1$, and $70\% / 40\%$ train and test accuracy under $\calA_2$. Then, this strategy would only train on samples from $\calA_1$, despite $\calA_2$ being stronger.
	
	\item \textbf{``Avg'' strategy:} This strategy simultaneously trains on adversarial examples from all attacks. That is, the $\max$ in $\hat{\calR}_{\text{adv}}$ is replaced by $\frac1n\sum_{i=1}^n L(f(\calA_i(\vec{x}), y))$. 
\end{myenumerate}

\paragraph{The sparse \texorpdfstring{$\ell_1$}{L1}-descent attack (SLIDE).}
%\label{ssec:l1-attack}
	\begin{figure}[t]
		\centering
		\small
		\begin{minipage}{\linewidth}
			\begin{algorithm}[H]
				\SetAlgoLined
				\DontPrintSemicolon
				\KwIn{Input $\vec{x} \in [0, 1]^d$, steps $k$, step-size $\gamma$, percentile $q$, $\ell_1$-bound $\epsilon$}
				\KwOut{$\hat{\vec{x}} = \vec{x}+\vec{r}$ s.t. $\norm{\vec{r}}_1 \leq \epsilon$}
				\vspace{0.3em}\hrule\vspace{0.3em}
				$\vec{r} \gets \mathbf{0}^d$\;
				\For{$1 \leq i \leq k$}{
					$\vec{g} \gets \nabla_\vec{r} L(\theta, \vec{x} + \vec{r}, y)$\;
					%$e_i \gets \begin{cases} \sign(g_i) & \text{if } |g_i| \geq P_q(|\vec{g}|) \\ 0 & \text{otherwise} \end{cases}$\;
					$e_i = \sign(g_i) \text{ if } |g_i| \geq P_q(|\vec{g}|), \text{ else } 0$\;
					$\vec{r} \gets \vec{r} + \gamma \cdot {\vec{e}}/{\norm{\vec{e}}_1}$\;
					$\vec{r} \gets \Pi_{S_{1}^{\epsilon}}(\vec{r})$\;
				}
				\vspace{0.3em}\hrule\vspace{0.3em}
				\caption{\textbf{The Sparse $\ell_1$ Descent Attack (SLIDE).} $P_q(|\vec{g}|)$ denotes the $q$\textsuperscript{th} percentile of $|\vec{g}|$ and $\Pi_{S_{1}^{\epsilon}}$ is the projection onto the $\ell_1$-ball (see~\cite{duchi2008efficient}).}
				\label{alg:pgd}
			\end{algorithm}
		\end{minipage}
\end{figure}

Adversarial training is contingent on a \emph{strong} and \emph{efficient} attack. Training on weak attacks gives no robustness~\cite{tramer2018ensemble}, while strong optimization attacks (e.g.,~\cite{carlini2016towards, chen2018ead}) are prohibitively expensive. Projected Gradient Descent (PGD)~\cite{kurakin2016scale, madry2018towards}
is a popular choice of attack that is both efficient and produces strong perturbations. To complement our formal results, we want to train models on $\ell_1$-perturbations. Yet, we show that the $\ell_1$-version of PGD is highly inefficient, and propose a better approach suitable for adversarial training.

%Projected Gradient Descent (PGD)~\cite{kurakin2016scale, madry2018towards} is a \emph{steepest descent} method (see~\cite{madry2018tutorial}). For an $\ell_p$ attack, where $S_{p}^{\epsilon}$ is the $\ell_p$-ball of radius $\epsilon$, each PGD iteration updates the perturbation $\vec{r}$ as
%
%\begin{equation*}
%\vec{r} '= \Pi_{S_{p}^{\epsilon}}\Big(\vec{r} + \gamma \cdot \argmax_{\norm{\vec{v}}_p \leq 1} \vec{v}^T \vec{g}\Big),\text{where } \vec{g} = \nabla_{\vec{r}} L(\theta, \vec{x} + \vec{r}, y), \Pi_{S}(\vec{r}) = \argmin_{\vec{r} \in S} \norm{\vec{r}' -\vec{r}}_2 \;.
%\end{equation*}
%
%That is, we take steps (with step-size $\gamma$) in the \emph{steepest descent direction} under the chosen $\ell_p$-norm, and then project the perturbation onto the closest point (under the Eulcidean metric) in the $\ell_p$-ball. For the $\ell_\infty$-norm, the steepest descent is $\texttt{sign}(\vec{g})$ and the projection clips $\vec{r}$ to the range $[-\epsilon, \epsilon]$. For $\ell_2$, the steepest descent is $\vec{g}/{\norm{\vec{g}}_2}$ and the projection maps an $\vec{r}$ outside of $S_{2}^{\epsilon}$ to $\epsilon \cdot {\vec{r}}/{\norm{\vec{r}}_2}$. 

PGD is a \emph{steepest descent} algorithm~\cite{madry2018tutorial}. In each iteration, the perturbation is updated in the steepest descent direction $\argmax_{\norm{\vec{v}} \leq 1} \vec{v}^T \vec{g}$, where $\vec{g}$ is the gradient of the loss. 
For the $\ell_\infty$-norm, the steepest descent direction is $\texttt{sign}(\vec{g})$~\cite{goodfellow2014explaining}, and for $\ell_2$, it is $\vec{g}/{\norm{\vec{g}}_2}$. For the $\ell_1$-norm, the steepest descent direction is the unit vector $\vec{e}$ with $e_{i^*} = \sign(g_{i^*})$, for $i^* = \argmax_i |g_i|$.

This yields an inefficient attack, as each iteration updates a single index of the perturbation $\vec{r}$. We thus design a new attack with finer control over the sparsity of an update step.
For $q \in [0,1]$, let $P_q(|\vec{g}|)$ be the $q$\textsuperscript{th} \emph{percentile} of $|\vec{g}|$.
We set $e_i = \sign(g_i)$ if $|g_i| \geq P_q(|\vec{g}|)$ and $0$ otherwise, and normalize $\vec{e}$ to unit $\ell_1$-norm. For $q \gg 1/d$, we thus update many indices of $\vec{r}$ at once. We introduce another optimization to handle clipping, by ignoring gradient components where the update step cannot make progress (i.e., where $x_i + r_i \in \{0,1\}$ and $g_i$ points outside the domain).
To project $\vec{r}$ onto an $\ell_1$-ball, we use an algorithm of Duchi et al.~\cite{duchi2008efficient}.
%(Duchi et al.~also give a $O(d)$ algorithm, but we prefer the $O(d \log{d})$ variant in practice as it allows for a more efficient batched implementation.) 
Algorithm~\ref{alg:pgd} describes our attack. It outperforms the steepest descent attack as well as a recently proposed Frank-Wolfe algorithm for $\ell_1$-attacks~\cite{kang2019transfer} (see Appendix~\ref{apx:pgd-l1}). Our attack is competitive with the more expensive EAD attack~\cite{chen2018ead} (see Appendix~\ref{apx:results}).

\section{Experiments}
\label{sec:experiments}

We use our new adversarial training schemes to measure the robustness trade-off on MNIST and CIFAR10.%
\footnote{Kang et al.~\cite{kang2019transfer} recently studied the transfer between $\ell_\infty, \ell_1$ and $\ell_2$-attacks for adversarially trained models on ImageNet. They show that models trained on one type of perturbation are not robust to others, but they do not attempt to train models against multiple attacks simultaneously.}
MNIST is an interesting case-study as \emph{distinct} models achieve strong robustness to different $\ell_p$ and spatial attacks\cite{schott2018towards, engstrom2017rotation}. Despite the dataset's simplicity, we show that no single model achieves strong $\ell_\infty, \ell_1$ and $\ell_2$ robustness, and that new techniques are required to close this gap. 
The code used for all of our experiments can be found here: \url{https://github.com/ftramer/MultiRobustness}

\paragraph{Training and evaluation setup.}
We first use adversarial training to train models on a single perturbation type. For MNIST,  we use $\ell_1 (\epsilon=10)$, $\ell_2 (\epsilon=2)$ and $\ell_\infty (\epsilon=0.3)$. For CIFAR10 we use $\ell_\infty (\epsilon=\frac{4}{255})$ and $\ell_1 (\epsilon=\frac{2000}{255})$. We also train on rotation-translation attacks with $\pm3$px translations and $\pm30\degree$ rotations as in~\cite{engstrom2017rotation}. We denote these models Adv$_1$, Adv$_2$, Adv$_\infty$, and Adv$_{\text{RT}}$.  
We then use the ``max'' and ``avg'' strategies from Section~\ref{sec:adv-training} to train models Adv$_\text{max}$ and  Adv$_\text{avg}$ against multiple perturbations. We train once on all $\ell_p$-perturbations, and once on both $\ell_\infty$ and RT perturbations.
We use the same CNN (for MNIST) and wide ResNet model (for CIFAR10) as Madry et al.~\cite{madry2018towards}. Appendix~\ref{apx:setup} has more details on the training setup, and attack and training hyper-parameters.

We evaluate robustness of all models using multiple attacks: (1) we use \emph{gradient-based attacks} for all $\ell_p$-norms, i.e., PGD~\cite{madry2018towards} and our SLIDE attack with $100$ steps and $40$ restarts ($20$ restarts on CIFAR10), as well as Carlini and Wagner's $\ell_2$-attack~\cite{carlini2016towards} (C\&W), and an $\ell_1$-variant---EAD~\cite{chen2018ead}; (2) to detect gradient-masking, we use \emph{decision-based attacks}: the Boundary Attack~\cite{brendel2018decision} for $\ell_2$, the Pointwise Attack~\cite{schott2018towards} for $\ell_1$, and the Boundary Attack++~\cite{chen2019boundary} for $\ell_\infty$; (3) for spatial attacks, we use the {optimal attack} of~\cite{engstrom2017rotation} that enumerates all small rotations and translations. For unbounded attacks (C\&W, EAD and decision-based attacks), we discard perturbations outside the $\ell_p$-ball.

%As we will see, evaluating against this wide-range of attacks proves useful in uncovering pernicious gradient-masking phenomena. We will also showcase the viability of our novel $\ell_1$-PGD attack---compared to more expensive gradient-based attacks such as EAD.

For each model, we report accuracy on $1000$ test points for: (1) individual perturbation types; (2) the union of these types, i.e., $1-\calR_{\text{adv}}^{\text{max}}$; and (3) the average of all perturbation types, $1-\calR_{\text{adv}}^{\text{avg}}$.
We briefly discuss the optimal error that can be achieved if there is no robustness trade-off. For perturbation sets $S_1, \dots S_n$, let $\calR_1, \dots, \calR_n$ be the optimal risks achieved by distinct models. Then, a single model can at best achieve risk $\calR_i$ for each $S_i$, i.e., $\OPT(\calR_{\text{adv}}^{\text{avg}}) = \frac{1}{n}\sum_{i=1}^n\calR_i$. 
If the errors are fully correlated, so that a maximal number of inputs admit \emph{no} attack, we have $\OPT(\calR_{\text{adv}}^{\text{max}}) = \max \{\calR_1, \dots, \calR_n\}$. Our experiments show that these optimal error rates are not achieved.

\paragraph{Results on MNIST.}

Results are in Table~\ref{tab:mnist_results}. The left table is for the union of $\ell_p$-attacks, and the right table is for the union of $\ell_\infty$ and RT attacks. In both cases, the multi-perturbation training strategies ``succeed'', in that models Adv$_\text{avg}$ and Adv$_\text{max}$ achieve higher multi-perturbation accuracy than any of the models trained against a single perturbation type. 

The results for $\ell_\infty$ and RT attacks are promising, although the best model Adv$_\text{max}$ only achieves $1-\calR_{\text{adv}}^{\text{max}}=83.8\%$ and $1-\calR_{\text{adv}}^{\text{avg}}=87.6\%$, which is far less than the optimal values, $1-\OPT(\calR_{\text{adv}}^{\text{max}})=\min\{91.4\%, 94.6\%\}=91.4\%$ and $1-\OPT(\calR_{\text{adv}}^{\text{avg}})=(91.4\% + 94.6\%)/2 = 93\%$. Thus, these models do exhibit some form of the robustness trade-off analyzed in Section~\ref{sec:theory}.

\begin{table}[t]
	\caption{\textbf{Evaluation of MNIST models trained on $\ell_\infty, \ell_1$ and $\ell_2$ attacks (left) or $\ell_\infty$ and rotation-translation (RT) attacks (right).} Models Adv$_\infty$, Adv$_1$, Adv$_2$ and Adv$_{\text{RT}}$ are trained on a single attack, while Adv$_{\text{avg}}$ and Adv$_{\text{max}}$ are trained on multiple attacks using the ``avg'' and ``max'' strategies. 
	The columns show a model's accuracy on individual perturbation types, on the union of them ($1-\calR_{\text{adv}}^{\text{max}}$), and the average accuracy across them ($1-\calR_{\text{adv}}^{\text{avg}}$).
	The best results are in bold (at $95\%$ confidence). Results in red indicate gradient-masking, see Appendix~\ref{apx:results} for a breakdown of all attacks.
	\ifnips\\[-0.5em]\else\fi
	}
	\label{tab:mnist_results}
	\centering
	\small
	%\ifnips
	%\setlength{\tabcolsep}{4pt}
	%\else
	\setlength{\tabcolsep}{3pt}
	%\fi
	\renewcommand{\arraystretch}{0.9}
	\begin{minipage}{0.51\textwidth}
		\centering
		\begin{tabular}{@{}l@{\hskip 4pt}r r r r a a @{}}
			Model & Acc.	& \multicolumn{1}{c}{$\ell_\infty$} & \multicolumn{1}{c}{$\ell_1$}& \multicolumn{1}{c}{$\ell_2$} & $1-\calR_{\text{adv}}^{\text{max}}$ & $1-\calR_{\text{adv}}^{\text{avg}}$ \\
			\toprule
			Nat 					& \textbf{99.4} &   0.0 &12.4 & 8.5 & 0.0 & 7.0\\
			\addlinespace
			Adv$_\infty$		& \textbf{99.1}  & \textbf{91.1} & {\color{red}12.1} & {\color{red}11.3} & 6.8 & 38.2\\
			Adv$_1$				& 98.9 &  0.0 & \textbf{78.5} & 50.6 & 0.0
			& 43.0\\
			Adv$_2$				& 98.5 & 0.4 & 68.0 & \textbf{71.8} & 0.4
			& 46.7\\
			\addlinespace
			Adv$_{\text{avg}}$ & 97.3 & 76.7 & {\color{red}53.9} & {\color{red}58.3} & \textbf{49.9} & \textbf{63.0}\\
			Adv$_{\text{max}}$ & 97.2 & 71.7 & {\color{red}62.6} & {\color{red}56.0} & \textbf{52.4} & \textbf{63.4}\\
		\end{tabular}
	\end{minipage}
	\hfill
	\begin{minipage}{0.45\textwidth}
		\renewcommand{\arraystretch}{0.9}
		\begin{tabular}{@{}l@{\hskip 4pt}r r r a a @{}}
			Model & Acc.	& \multicolumn{1}{c}{$\ell_\infty$} & \multicolumn{1}{c}{RT} & $1-\calR_{\text{adv}}^{\text{max}}$ & $1-\calR_{\text{adv}}^{\text{avg}}$ \\
			\toprule
			Nat & \textbf{99.4} & 0.0 & 0.0 & 0.0 & 0.0\\
			\addlinespace
			Adv$_\infty$ & \textbf{99.1} & \textbf{91.4} & 0.2 & 0.2 & 45.8\\
			Adv$_{\text{RT}}$ & \textbf{99.3} & 0.0 & \textbf{94.6} & 0.0 & 47.3\\
			\addlinespace
			Adv$_{\text{avg}}$ & \textbf{99.2} & 88.2 & 86.4 & \textbf{82.9} & \textbf{87.3}\\
			Adv$_{\text{max}}$ & 98.9 & 89.6 & 85.6 & \textbf{83.8} & \textbf{87.6}\\
			\\
		\end{tabular}
	\end{minipage}
	%\ifnips\vspace{-1.5em}\else\vspace{-0.5em}\fi
	\vspace{-0.5em}
\end{table}

The $\ell_p$ results are surprisingly mediocre and re-raise questions about whether MNIST can be considered ``solved'' from a robustness perspective. Indeed, while training \emph{separate} models to resist $\ell_1, \ell_2$ or $\ell_\infty$ attacks works well, resisting all attacks simultaneously fails. This agrees with the results of Schott et al.~\cite{schott2018towards}, whose models achieve either high $\ell_\infty$ or $\ell_2$ robustness, but not both simultaneously. We show that in our case, this lack of robustness is partly due to gradient masking.

\paragraph{First-order adversarial training and gradient masking on MNIST.}
The model Adv$_\infty$ is not robust to $\ell_1$ and $\ell_2$-attacks. This is unsurprising as the model was only trained on $\ell_\infty$-attacks. Yet, comparing the model's accuracy against multiple types of $\ell_1$ and $\ell_2$ attacks (see Appendix~\ref{apx:results}) reveals a more curious phenomenon: Adv$_\infty$ has high accuracy against \emph{first-order} $\ell_1$ and $\ell_2$-attacks such as PGD, but is broken by decision-free attacks. This is an indication of gradient-masking~\cite{papernot2016practical,tramer2018ensemble,athalye2018obfuscated}.

This issue had been observed before~\cite{schott2018towards, li2018second}, but an explanation remained illusive, especially since $\ell_\infty$-PGD does not appear to suffer from gradient masking (see~\cite{madry2018towards}). We explain this phenomenon by
inspecting the learned features of model Adv$_\infty$, as in~\cite{madry2018towards}. We find that the model's first layer learns threshold filters $\vec{z}=\relu(\alpha\cdot(\vec{x}-\epsilon))$ for $\alpha > 0$.
As most pixels in MNIST are zero, most of the $z_i$ cannot be activated by an $\epsilon$-bounded $\ell_\infty$-attack. The $\ell_\infty$-PGD thus 
optimizes a smooth (albeit flat) loss function. In contrast, $\ell_1$- and $\ell_2$-attacks can move a pixel $x_i=0$ to $\hat{x}_i > \epsilon$ thus activating $z_i$, but have no gradients to rely on (i.e, $\mathrm{d} z_i/ \mathrm{d}{x}_i = 0$ for any ${x}_i \leq \epsilon$). Figure~\ref{fig:masking} in Appendix~\ref{apx:masking} shows that the model's loss resembles a step-function, for which first-order attacks such as PGD are inadequate.

Note that training against first-order $\ell_1$ or $\ell_2$-attacks directly (i.e., models Adv$_1$ and Adv$_2$ in Table~\ref{tab:mnist_results}), seems to yield genuine robustness to these perturbations. This is surprising in that, because of gradient masking, model Adv$_\infty$ actually achieves lower training loss against first-order $\ell_1$ and $\ell_2$-attacks than models Adv$_1$ and Adv$_2$. That is, Adv$_1$ and Adv$_2$ converged to sub-optimal local minima of their respective training objectives, yet these minima generalize much better to stronger attacks.

The models Adv$_\text{avg}$ and Adv$_\text{max}$ that are trained against $\ell_\infty, \ell_1$ and $\ell_2$-attacks also learn to use thresholding to resist $\ell_\infty$-attacks while spuriously masking gradient for $\ell_1$ and $\ell_2$-attacks. This is evidence that, unlike previously thought~\cite{tsipras2019robustness}, training against a strong first-order attack (such as PGD) can cause the model to minimize its training loss via gradient masking. To circumvent this issue, alternatives to first-order adversarial training seem necessary. Potential (costly) approaches include training on gradient-free attacks, or extending certified defenses~\cite{raghunathan2018certified,wong2018provable} to multiple perturbations. Certified defenses provide provable bounds that are much weaker than the robustness attained by adversarial training, and certifying multiple perturbation types is likely to exacerbate this gap.

\paragraph{Results on CIFAR10.}

The left table in Table~\ref{tab:cifar_results} considers the union of $\ell_\infty$ and $\ell_1$ perturbations, while the right table considers the union of $\ell_\infty$ and RT perturbations.
As on MNIST, the models Adv$_{\text{avg}}$ and Adv$_{\text{max}}$ achieve better multi-perturbation robustness than any of the models trained on a single perturbation, but fail to match the optimal error rates we could hope for. For $\ell_1$ and $\ell_\infty$-attacks, we achieve $1-\calR_{\text{adv}}^{\text{max}}=61.1\%$ and $1-\calR_{\text{adv}}^{\text{avg}}=64.1\%$, again significantly below the optimal values, $1-\OPT(\calR_{\text{adv}}^{\text{max}})=\min\{71.0\%, 66.2\%\}=66.2\%$ and $1-\OPT(\calR_{\text{adv}}^{\text{avg}})=(71.0\% + 66.2\%)/2 = 68.6\%$. The results for $\ell_\infty$ and RT attacks are qualitatively and quantitatively similar. %
\footnote{An interesting open question is why the model Adv$_\text{avg}$ trained on $\ell_\infty$ and RT attacks does not attain optimal average robustness $\calR_{\text{adv}}^{\text{avg}}$. Indeed, on CIFAR10, detecting the RT attack of~\cite{engstrom2017rotation} is easy, due to the black in-painted pixels in a transformed image. The following ``ensemble'' model thus achieves optimal $\calR_{\text{adv}}^{\text{avg}}$ (but not necessarily optimal $\calR_{\text{adv}}^{\text{max}}$): on input $\hat{\vec{x}}$, return Adv$_{\text{RT}}(\hat{\vec{x}})$ if there are black in-painted pixels, otherwise return Adv$_\infty(\hat{\vec{x}})$. The fact that model Adv$_\text{avg}$ did not learn such a function might hint at some limitation of adversarial training.
}

\begin{table}[t]
	\centering
	\setlength{\tabcolsep}{3pt}
	\renewcommand{\arraystretch}{0.9}
	
	\caption{\textbf{Evaluation of CIFAR10 models trained against $\ell_\infty$ and $\ell_1$ attacks (left) or $\ell_\infty$ and rotation-translation (RT) attacks (right).} Models Adv$_\infty$, Adv$_1$ and Adv$_{\text{RT}}$ are trained against a single attack, while Adv$_{\text{avg}}$ and Adv$_{\text{max}}$ are trained against two attacks using the ``avg'' and ``max'' strategies. 
	The columns show a model's accuracy on individual perturbation types, on the union of them ($1-\calR_{\text{adv}}^{\text{max}}$), and the average accuracy across them ($1-\calR_{\text{adv}}^{\text{avg}}$).	The best results are in bold (at $95\%$ confidence). A breakdown of all $\ell_1$ attacks is in Appendix~\ref{apx:results}.
	\ifnips\\[-0.5em]\else\fi
}
	\label{tab:cifar_results}
	
	\begin{minipage}{0.48\textwidth}
		\small
		\centering
		\begin{tabular}{@{}l r r r a a @{}}
			%&&\multicolumn{1}{c}{$\ell_\infty (\frac{4}{255})$} & \multicolumn{1}{c}{$\ell_1 (\frac{2000}{255})$} \\
			%\cmidrule(lr){3-3} \cmidrule(lr){4-4}
			Model & Acc.	& \multicolumn{1}{c}{$\ell_\infty$} & 
			%PGD & EAD & PA$\ $ & 
			\multicolumn{1}{c}{$\ell_1$} & $1-\calR_{\text{adv}}^{\text{max}}$ & $1-\calR_{\text{adv}}^{\text{avg}}$ \\
			\toprule
			Nat							& \textbf{95.7} & 0.0 & 
			%0.2 & 0.0 & 29.6 & 
			0.0 & 0.0 & 0.0\\
			\addlinespace
			Adv$_{\infty}$			& 92.0 & \textbf{71.0} & 
			%19.4 & 17.6 & 52.7 & 
			16.4 & 16.4 & 44.9\\
			Adv$_{1}$	 			& 90.8 & 53.4 & 
			%66.6 & 66.6 & 84.7 & 
			\textbf{66.2} & 53.1& 60.0 \\
			\addlinespace
			Adv$_{\text{avg}}$	 & 91.1	 & 64.1	& 60.8 & \textbf{59.4}	& \textbf{62.5} \\
			Adv$_{\text{max}}$	& 91.2 & 65.7 & 62.5 & \textbf{61.1} & \textbf{64.1}
		\end{tabular}
	\end{minipage}
	\hfill
	\begin{minipage}{0.48\textwidth}
		\centering
		\small
		\begin{tabular}{@{}l r r r a a @{}}
			%\vphantom{$\ell_\infty (\frac{4}{255})$}\\
			%\addlinespace
			Model & Acc.	& \multicolumn{1}{c}{$\ell_\infty$} &\multicolumn{1}{c}{RT} & $1-\calR_{\text{adv}}^{\text{max}}$ & $1-\calR_{\text{adv}}^{\text{avg}}$ \\
			\toprule
			Nat							& \textbf{95.7} & 0.0 & 5.9 & 0.0 & 3.0\\
			\addlinespace
			Adv$_{\infty}$			& 92.0 & \textbf{71.0} & 8.9 & 8.7 & 40.0 \\
			Adv$_{\text{RT}}$	 & \textbf{94.9} & 0.0 & \textbf{82.5} & 0.0 & 41.3 \\
			\addlinespace
			Adv$_{\text{avg}}$	 & 93.6 & 67.8  & 78.2 & \textbf{65.2} & \textbf{73.0}  \\		
			Adv$_{\text{max}}$	& 93.1 & \textbf{69.6} & 75.2 & \textbf{65.7} & \textbf{72.4} \\
		\end{tabular}
	\end{minipage}
	%\ifnips\vspace{-1.25em}\fi
\end{table}

Interestingly, models Adv$_\text{avg}$ and Adv$_\text{max}$ achieve $100\%$ \emph{training accuracy}. Thus, multi-perturbation robustness increases the \emph{adversarial generalization gap}~\cite{schmidt2018adversarially}. These models might be resorting to more memorization because they fail to find features robust to both attacks.

\paragraph{Affine Adversaries.}
%\label{ssec:experiment_affine}

%In Section~\ref{sec:}, we showed that for a simple data distribution, affine combinations of $\ell_p$ perturbations were no stronger than the original $\ell_p$ perturbations, unless the target model exhibits some form of non-linearity. In contrast, we showed that affine combinations of $\ell_\infty$ and RT perturbations could be stronger than either individual attack even in the case of linear models.

%\ifnips
%\begin{wraptable}{r}{6.33cm}
	%\vspace{-1.25em}
%\else
\begin{table}[t]
%\fi
	\caption{\textbf{Evaluation of affine attacks.} For models trained with the ``max'' strategy, we evaluate against attacks from a union $S_U$ of perturbation sets, and against an affine adversary that interpolates between perturbations. Examples of affine attacks are in Figure~\ref{fig:affine_attacks}.
	\ifnips\\[-0.5em]\else\fi
}
	\label{tab:cifar_affine}
	\centering
	\small
	\setlength{\tabcolsep}{4pt}
	\renewcommand{\arraystretch}{0.9}
	\begin{tabular}{@{}l l r r @{}}
	Dataset & Attacks & acc.~on $S_{\text{U}}$ & 
	acc.~on $S_\text{affine}$ \\
	\toprule
		MNIST & $\ell_\infty$ \& RT & 83.8 & {62.6} \\
		%\addlinespace
		CIFAR10 & $\ell_\infty$ \& RT & 65.7 & {56.0}\\
		%\addlinespace
		CIFAR10 & $\ell_\infty$ \& $\ell_1$ & {61.1} & {58.0}
	\end{tabular}
%\ifnips
%\vspace{-1em}
%\end{wraptable}
%\else
\vspace{-0.5em}
\end{table}
%\fi
%
Finally, we evaluate the affine attacks introduced in Section~\ref{ssec:affine}. These attacks take affine combinations of two perturbation types, and we apply
them on the models Adv$_\text{max}$ (we omit the $\ell_p$-case on MNIST due to gradient masking). 
To compound $\ell_\infty$ and $\ell_1$-noise, we devise an attack that updates both perturbations in alternation. To compound $\ell_\infty$ and RT attacks, we pick random rotation-translations (with $\pm3\beta$px translations and $\pm30\beta\degree$ rotations), apply an $\ell_\infty$-attack with budget $(1-\beta)\epsilon$ to each, and retain the worst example. 

The results in Table~\ref{tab:cifar_affine} match the predictions of our formal analysis: (1) affine combinations of $\ell_p$ perturbations are no stronger than their union. This is expected given Claim~\ref{thm:linear_affine} and prior observations that neural networks are close to linear near the data~\cite{goodfellow2014explaining, ribeiro2016should}; (2) combining of $\ell_\infty$ and RT attacks does yield a stronger attack, as shown in Theorem~\ref{thm:affine_RT}.
This demonstrates that robustness to a union of perturbations can still be insufficient to protect against more complex combinations of perturbations.

\section{Discussion and Open Problems}

Despite recent success in defending ML models against some perturbation types~\cite{madry2018towards, engstrom2017rotation, schott2018towards}, extending these defenses to multiple perturbations unveils a clear robustness trade-off.
This tension may be rooted in its unconditional occurrence in natural and simple distributions, as we proved in Section~\ref{sec:theory}. 

Our new adversarial training strategies fail to achieve competitive robustness to more than one attack type, but narrow the gap towards multi-perturbation robustness. 
%Interestingly, our ``max'' strategy always outperforms the ``avg'' strategy, despite discarding most adversarial examples. 
We note that the optimal risks $\calR_{\text{adv}}^{\text{max}}$ and $\calR_{\text{adv}}^{\text{avg}}$ that we achieve are very close. Thus, for most data points, the models are either robust to all perturbation types or none of them. This hints that some points (sometimes referred to as \emph{prototypical examples}~\cite{carlini2018prototypical, stock2018convnets}) are inherently easier to classify robustly, regardless of the perturbation type.

We showed that first-order adversarial training for multiple $\ell_p$-attacks suffers from gradient masking on MNIST. Achieving better robustness on this simple dataset is an open problem.
Another challenge is reducing the cost of our adversarial training strategies, which scale linearly in the number of perturbation types. Breaking this linear dependency requires efficient techniques for finding perturbations in a union of sets, which might be hard for sets with near-empty intersection (e.g., $\ell_\infty$ and $\ell_1$-balls). 
The cost of adversarial training has also be reduced by merging the inner loop of a PGD attack and gradient updates of the model parameters~\cite{shafahi2019free,zhang2019you}, but it is unclear how to extend this approach to a union of perturbations (some of which are not optimized using PGD, e.g., rotation-translations).

Hendrycks and Dietterich~\cite{hendrycks2018benchmarking}, and Geirhos et al.~\cite{geirhos2018imagenet} recently measured robustness of classifiers to multiple common (i.e., non-adversarial) image corruptions (e.g., random image blurring). In that setting, they also find that different classifiers achieve better robustness to some corruptions, and that no single classifier achieves the highest accuracy under all forms. The interplay between multi-perturbation robustness in the adversarial and common corruption case is worth further exploration.

%% file: apx_experiments.tex
\section{Experimental Setup}
\label{apx:setup}

\paragraph{MNIST.}
We use the CNN model from Madry et al.~\cite{madry2018towards} and train for $10$ epochs with Adam and a learning rate of $10^{-3}$ reduced to $10^{-4}$ after $5$ epochs (batch size of $100$). To accelerate convergence, we train against a weaker adversary in the first epoch (with $1/3$ of the perturbation budget).
For training, we use PGD with $40$ iterations for $\ell_\infty$ and $100$ iterations for $\ell_1$ and $\ell_2$. For rotation-translations, we use the attack from~\cite{engstrom2017rotation} that picks the worst of $10$ random rotation-translations.

\paragraph{CIFAR10.}
We use the same wide ResNet model as~\cite{madry2018towards}. We train for $80$k steps of gradient descent with batch size $128$ ($205$ epochs). When using the ``avg'' strategy for wide ResNet models, we had to halve the batch size to avoid overflowing the GPU's memory. We accordingly doubled the number of training steps and learning rate schedule. We use a learning rate of $0.1$ decayed by a factor $10$ after $40$k and $60$k steps, a momentum of $0.9$, and weight decay of $0.0002$. Except for the RT attack, we use standard data augmentation with random padding, cropping and horizontal flipping (see~\cite{engstrom2017rotation} for details). We extract $1{,}000$ points from the CIFAR10 test as a validation set for early-stopping.

For training, we use PGD with $10$ iterations for $\ell_\infty$, and $20$ iterations for $\ell_1$. %
\footnote{Our new attack $\ell_1$-attack, described in Section~\ref{sec:adv-training}, has a parameter $q$ to controls the sparsity of the gradient updates. When leaving this parameter constant during training, the model overfits and fails to achieve general robustness. To resolve this issue, we sample $q \in [80\%, 99.5\%]$ at random for each attack during training. We also found that $10$ iterations were insufficient to get a strong attack and thus increased the iteration count to $20$.}
For rotation-translations, we also use the attack from~\cite{engstrom2017rotation} that trains on the worst of $10$ randomly chosen rotation-translations.

%For the ``max'' strategy, we choose the worst attack independently for each image in a training batch. That is, each adversarial batch can contain adversarial examples from different perturbation models.

%% file: apx_l1attack.tex
\section{Performance of the Sparse \texorpdfstring{$\ell_1$}{L1}-Descent Attack}
\label{apx:pgd-l1}

In Figure~\ref{fig:pgd_l1}, we compare the performance of our new Sparse $\ell_1$-Descent Attack (SLIDE) for different choices of gradient sparsity. We also compare to the standard PGD attack with the steepest-descent update rule, as well as a recent attack proposed in~\cite{kang2019transfer} that adapts the Frank-Wolfe optimization algorithm for finding $\ell_1$-bounded adversarial examples. As we explained in Section~\ref{sec:adv-training}, we expect our attack to outperform PGD as the steepest-descent vector is too sparse in the $\ell_1$-case, and we indeed observe a significant improvement by choosing denser updates.

The subpar performance of the Frank-Wolfe algorithm is also intriguing. We believe it is due to the attack's linearly decreasing step-size (the $k$\textsuperscript{th} iteration has a step-size of $O(1/k)$, see~\cite{kang2019transfer} for details). While this choice is appropriate for optimizing convex functions, in the non-convex case it overly emphasizes the first steps of the attack, which intuitively should increase the likelihood of landing in a local minima. 

\begin{figure}[H]
	\centering
	\begin{subfigure}{0.48\textwidth}
		\includegraphics[width=\textwidth]{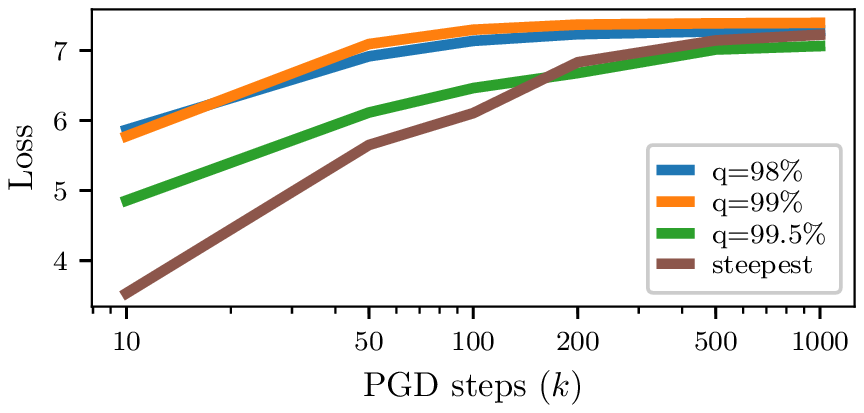}
	\end{subfigure}
	\hfill
	\begin{subfigure}{0.48\textwidth}
		\includegraphics[width=\textwidth]{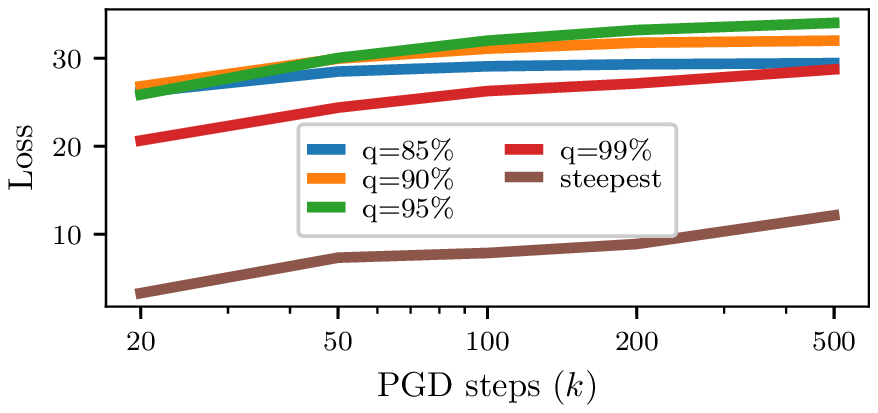}
	\end{subfigure}

	\caption{\textbf{Performance of the Sparse $\ell_1$-Descent Attack on MNIST (left) and CIFAR10 (right) for different choices of descent directions.} We run the attack for up to $1{,}000$ steps and plot the evolution of the cross-entropy loss, for an undefended model. We vary the sparsity of the gradient updates (controlled by the parameter $q$), and compare to the standard PGD attack that uses the steepest descent vector, as well as the Frank-Wolfe $\ell_1$-attack from~\cite{kang2019transfer}. For appropriate $q$, our attack vastly outperforms PGD and Frank-Wolfe.}
	\label{fig:pgd_l1}
\end{figure}

%% file: apx_results.tex
\newpage
\section{Breakdown of \texorpdfstring{$\ell_p$}{LP}-Attacks on Adversarially Trained Models}
\label{apx:results}

Tables~\ref{tab:mnist_lps_detailed} and~\ref{tab:cifar_lps_detailed} below give a more detailed breakdown of each model's accuracy against each $\ell_p$ attack we considered.  For each model and attack, we evaluate the attack on $1{,}000$ test points and report the accuracy. For each individual perturbation type (i.e., $\ell_\infty, \ell_1, \ell_2$), we further report the accuracy obtained by choosing the worst attack for each input. Finally, we report the accuracy against the union of all attacks ($1-\calR_{\text{adv}}^{\text{max}}$) as well as the average accuracy across perturbation types ($1-\calR_{\text{adv}}^{\text{avg}}$).

\begin{table}[h]
	\centering
	\addtolength{\leftskip} {-2cm}
	\addtolength{\rightskip}{-2cm}
	\caption{\textbf{Breakdown of all attacks on MNIST models.}
		For $\ell_\infty$, we use PGD and the Boundary Attack++ (BAPP)~\cite{chen2019boundary}. For $\ell_1$, we use our new Sparse $\ell_1$-Descent Attack (SLIDE), EAD~\cite{chen2018ead} and the Pointwise Attack (PA)~\cite{schott2018towards}. For $\ell_2$, we use PGD, C\&W~\cite{carlini2016towards} and the Boundary Attack (BA)~\cite{brendel2018decision}.\\[-0.5em]}
	\label{tab:mnist_lps_detailed}
	\small
	\setlength{\tabcolsep}{4pt}
	\renewcommand{\arraystretch}{0.9}
	\begin{tabular}{@{}l r rra rrra rrra b b @{}}
		&&\multicolumn{3}{c}{$\ell_\infty$} & \multicolumn{4}{c}{$\ell_1$} & \multicolumn{4}{c}{$\ell_2$} & \multicolumn{1}{c}{}\\
		\cmidrule(lr){3-5} \cmidrule(lr){6-9} \cmidrule(lr){10-13}
		Model & Acc. & PGD & BAPP & All $\ell_\infty$ & SLIDE & EAD & PA$\ $ & All $\ell_1$ & PGD & C\&W & BA$\ $ & All $\ell_2$ & $1-\calR_{\text{adv}}^{\text{max}}$ & $1-\calR_{\text{adv}}^{\text{avg}}$ \\
		\toprule
		Nat 					& \textbf{99.4} & 0.0 & 13.0 & 0.0 & 13.0 & 18.8 & 72.1 & 12.4 & 11.0 & 10.4 & 31.0 & 8.5 & 0.0 & 7.0\\
		\addlinespace
		Adv$_\infty$		& \textbf{99.1} & 91.1 & 98.5 & \textbf{91.1} & {66.9} & {58.4} & 15.0 & {\color{red}12.1} & {78.1} & {78.4} & 14.0 & {\color{red}11.3} & 6.8 & 38.2\\
		Adv$_1$				& 98.9 &  0.0 & 43.5  & 0.0 & 78.6 & 81.0 & 91.6 & \textbf{78.5} & 53.0 & 52.0 & 69.7 & 50.6 & 0.0
		& 43.0\\
		Adv$_2$				& 98.5 & 0.4 &  78.5 & 0.4 & 70.4 & 69.3 & 89.7 & 68.0 & 74.7 & 74.5 & 81.7 & \textbf{71.8} & 0.4
		& 46.7\\
		\addlinespace
		Adv$_{\text{avg}}$ & 97.3 & 76.7 & 98.0 & 76.7 & {66.3} & {62.4} & {68.6} & {\color{red}53.9} & {77.7} & {72.3} & {64.6} & {\color{red}58.3} & \textbf{49.9} & \textbf{63.0}\\
		Adv$_{\text{max}}$ & 97.2 & 71.7 & 98.5 & 71.7 & {72.1} & {70.0} & {69.6} & {\color{red}62.6} & {75.7} & {71.8} & {59.7} & {\color{red}56.0} & \textbf{52.4} & \textbf{63.4}\\
	\end{tabular}
	%\ifnips\vspace{-1.5em}\fi
\end{table}

\begin{table}[h]
	\caption{\textbf{Breakdown of all attacks on CIFAR10 models.}
		For $\ell_\infty$, we use PGD. For $\ell_1$, we use our new Sparse $\ell_1$-descent attack (SLIDE), EAD~\cite{chen2018ead} and the Pointwise Attack (PA)~\cite{schott2018towards}.\\[-0.5em]}
	\label{tab:cifar_lps_detailed}
	\centering
	\small
	\setlength{\tabcolsep}{4pt}
	\renewcommand{\arraystretch}{0.9}
	\begin{tabular}{@{}l r ra rrra b b @{}}
		&&\multicolumn{2}{c}{$\ell_\infty$} & \multicolumn{4}{c}{$\ell_1$}\\
		\cmidrule(lr){3-4} \cmidrule(lr){5-8} 
		Model & Acc.	& PGD & All $\ell_\infty$ & SLIDE & EAD & PA$\ $ & All $\ell_1$ & $1-\calR_{\text{adv}}^{\text{max}}$ & $1-\calR_{\text{adv}}^{\text{avg}}$ \\
		\toprule
		Nat	& \textbf{95.7} & 0.0 & 0.0 &
		0.2 & 0.0 & 29.6 & 
		0.0 & 0.0 & 0.0\\
		\addlinespace
		Adv$_{\infty}$			& 92.0 & 71.0 & \textbf{71.0} & 
		19.4 & 17.6 & 52.7 & 
		16.4 & 16.4 & 44.9\\
		Adv$_{1}$	 			& 90.8 & 53.4 & 53.4 &
		66.6 & 66.6 & 84.7 & 
		\textbf{66.2} & 53.1& 60.0 \\
		\addlinespace
		Adv$_{\text{avg}}$	 & 91.1	 & 64.1	& 64.1 &
		61.1& 61.5 & 81.7 &
		60.8 & \textbf{59.4}	& \textbf{62.5} \\
		Adv$_{\text{max}}$	& 91.2 & 65.7 & 65.7 &
		63.1 & 63.0 & 83.4 &
		62.5 & \textbf{61.1} & \textbf{64.1}
	\end{tabular}
	%\ifnips\vspace{-1em}\fi
\end{table}

%% file: apx_masking.tex
%\ifnips\else\newpage\fi
\section{Gradient Masking as a Consequence of \texorpdfstring{$\ell_\infty$}{Linf}-Robustness on MNIST.}
\label{apx:masking}

Multiple works have reported on a curious phenomenon that affects the $\ell_\infty$-adversarially trained model of Madry et al.~\cite{madry2018towards} on MNIST. This model achieves strong robustness to the $\ell_\infty$ attacks it was trained on, as one would expect. Yet, on other $\ell_p$-norms (e.g., $\ell_1$~\cite{chen2018ead, schott2018towards} and $\ell_2$~\cite{li2018second, schott2018towards}), its robustness is no better---or even worse---than for an undefended model. Some authors have referred to this effect as \emph{overfitting}, a somewhat unfair assessment of the work of~\cite{madry2018towards}, as their model actually achieves exactly what it was trained to do---namely resist $\ell_\infty$-bounded attacks. Moreover, as our theoretical results suggest, this trade-off may be inevitable (a similar point was made in~\cite{khoury2019geometry}). %To give credence to this statement, our experiments with MNIST in Section~\ref{} show that simultaneously achieving state-of-the-art robustness to $\ell_\infty, \ell_1$ and $\ell_2$ attacks seems hard with current techniques.

The more intriguing aspect of Madry et al.'s MNIST model is that, when attacked by $\ell_1$ or $\ell_2$ adversaries, first-order attacks are sub-optimal. This was previously observed in~\cite{schott2018towards} and in~\cite{li2018second}, where decision-based or second-order attacks vastly outperformed gradient descent for finding
$\ell_1$ or $\ell_2$ adversarial examples. Li et al.~\cite{li2018second} argue that this effect is due to the gradients of the adversarially trained model having much smaller magnitude than in a standard model. Yet, this fails to explain why first-order attacks appear to be optimal in the $\ell_\infty$-norm that the model was trained against.

A natural explanation for this discrepancy follows from an inspection of the robust model's first layer (as done in~\cite{madry2018towards}). All kernels of the model's first convolutional layer have very small norm, except for three kernels that have a single large weight. This reduces the convolution to a thresholding filter, which we find to be of one of two forms: either $\relu(\alpha \cdot (x - 0.3))$ or $\relu(\alpha \cdot (x - 0.7))$ for constant $\alpha > 0$.\footnote{Specifically, for the ``secret'' model of Madry et al., the three thresholding filters are approximately $\relu(0.6\cdot(x-0.3))$, $\relu(1.34\cdot(x-0.3))$ and $\relu(0.86\cdot(x-0.7))$.}
%Moreover, we observe that the kernels in the second convolutional layer that operate on a $\relu(\alpha \cdot (x - \epsilon))$ channel have negative weights, while the ones that operate on the $\relu(\alpha \cdot (x - (1-\epsilon)))$ channel have positive ones.
Thus, the model's first layer forms a piece-wise function with three distinct regimes, depending on the value of an input pixel $x_i$: (1) for $x_i \in [0, 0.3]$, the output is only influenced by the low-weight kernels. For $x_i \in [0.3, 1]$, the $\relu(\alpha \cdot (x - 0.3))$ filters become active, and override the signal from the low-weight kernels. For $x_i \in [0.7, 1]$, the $\relu(\alpha \cdot (x - 0.7))$ filters are also active.

As most MNIST pixels are in $\{0,1\}$, $\ell_\infty$-attacks operate
in a regime where most perturbed pixels are in $[0, 0.3] \cup [0.7, 1]$.
The model's large-weight ReLUs thus never transition between active and inactive, which leads to a smooth, albeit flat loss that first-order methods navigate effectively. 

For $\ell_1$ and $\ell_2$ attacks however, one would expect some of the ReLUs to be flipped as the attacks can make changes larger that $0.3$ to some pixels. Yet, as most MNIST pixels are $0$ (the digit's background), nearly all large-weight ReLUs start out inactive, with gradients equal to zero. 
A first-order adversary thus has no information on which pixels to focus the perturbation budget on.

Decision-based attacks sidestep this issue by disregarding gradients entirely. Figure~\ref{fig:masking} shows two examples of input points where a decision-based attack (Pointwise Attack for $\ell_1$~\cite{schott2018towards} and Boundary Attack for $\ell_2$~\cite{brendel2018decision}) finds an adversarial example in a direction that is orthogonal to the one explored by PGD. The loss surface exhibits sharp thresholding steps, as predicted by our analysis.

\begin{figure}[t]
	\centering
	\begin{subfigure}{0.4\textwidth}
		\includegraphics[width=\textwidth]{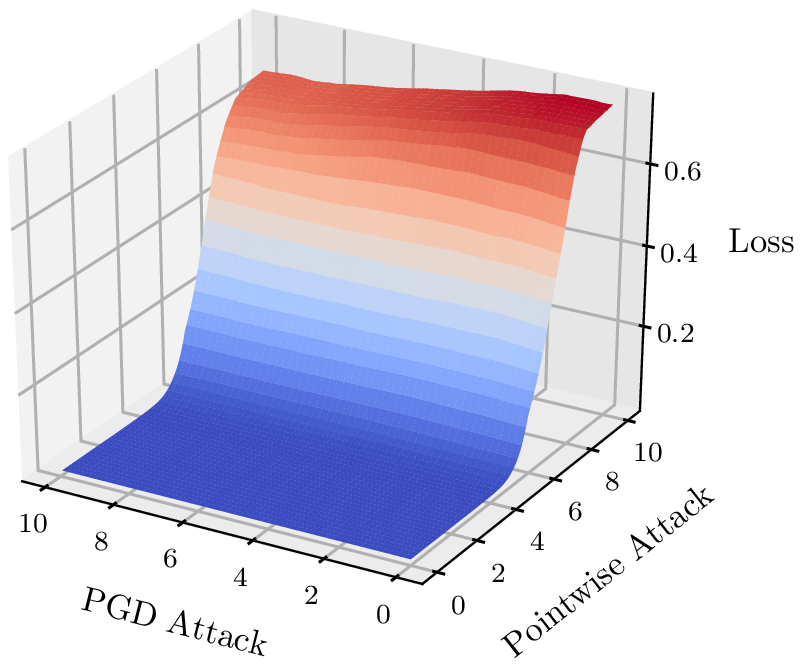}
	\end{subfigure}
	\quad\quad\quad\quad
	\begin{subfigure}{0.4\textwidth}
		\includegraphics[width=\textwidth]{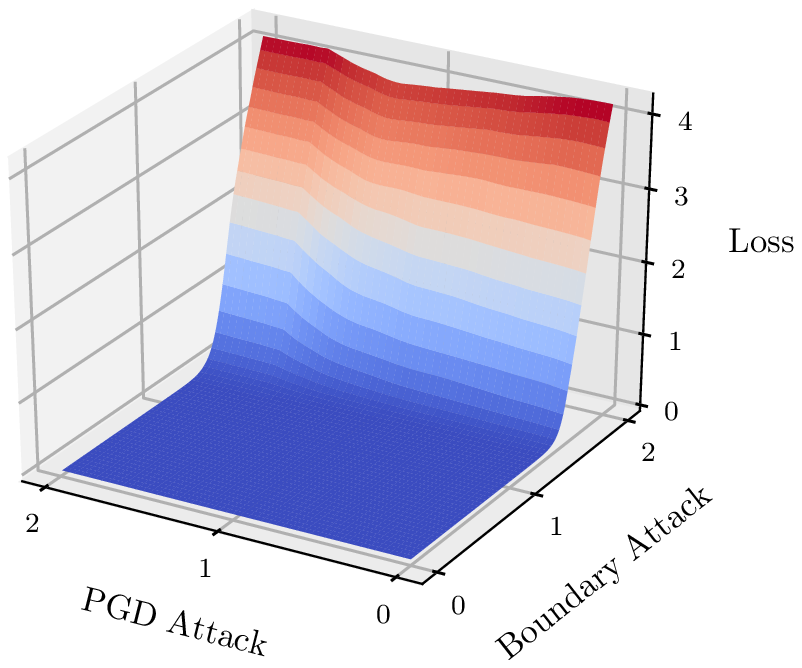}
	\end{subfigure}
	
	\caption{\textbf{Gradient masking in an $\ell_\infty$-adversarially trained model on MNIST, evaluated against $\ell_1$-attacks (left) and $\ell_2$-attacks (right).} The model is trained against an $\ell_\infty$-PGD adversary with $\epsilon=0.3$. For a randomly chosen data point $\vec{x}$, we compute an adversarial perturbation $\vec{r}_{\text{PGD}}$ using PGD and $\vec{r}_{\text{GF}}$ using a gradient-free attack. The left plot is for $\ell_1$-attacks with $\epsilon=10$ and the right plot is for $\ell_2$-attacks with $\epsilon=2$. The plots display the loss on points of the form $\hat{\vec{x}} \coloneqq \vec{x} + \alpha \cdot \vec{r}_{\text{PGD}} + \beta \cdot \vec{r}_{\text{GF}}$, for $\alpha, \beta \in [0, \epsilon]$. The loss surface behaves like a step-function, and gradient-free attacks succeed in finding adversarial examples where first-order methods failed.
	\ifnips\\[-1em]\else\\[-2em]\fi}
	\label{fig:masking}
	%\ifnips\vspace{-1em}\fi
\end{figure}

When we explicitly train against first-order $\ell_1$ or $\ell_2$ adversaries (models Adv$_1$ and Adv$_2$ in Table~\ref{tab:mnist_results}, left), the resulting model is robust (at least empirically) to $\ell_1$ or $\ell_2$ attacks. Note that model Adv$_\infty$ actually achieves higher robustness to $\ell_2$-PGD attacks than Adv$_2$ (due to gradient-masking). Thus, the Adv$_2$ model converged to a \emph{sub-optimal} local minima of its first-order adversarial training procedure (i.e., learning the same thresholding mechanism as Adv$_\infty$ would yield lower loss). Yet, this sub-optimal local minima generalizes much better to other $\ell_2$ attacks.

Models trained against $\ell_\infty, \ell_1$ and $\ell_2$ attacks (i.e., Adv$_\text{all}$ and Adv$_\text{max}$) in Table~\ref{tab:mnist_results}, left) also learn to use thresholding to achieve robustness to $\ell_\infty$ attacks, while masking gradients for $\ell_1$ and $\ell_2$ attacks.

%Moreover, the kernels in the second convolutional layer that operate on a $\relu(\alpha \cdot (x - \epsilon))$ channel have negative weights, while the ones that operate on the $\relu(\alpha \cdot (x - (1-\epsilon)))$ channel are positive. These observations let us draw a picture of the gradient of the model's (non-zero) second layer outputs with respect to an input pixel $x_i$: for $x_i \in [0, \epsilon]$, the gradient is very small as it is only influenced by the first layer's low-weight kernels. For $x_i \in [\epsilon, 1-\epsilon]$, the $\relu(\alpha \cdot (x - \epsilon))$ filters become active, and produce a constant negative gradient due to their second-layer weights. For $x_i \in [1-\epsilon]$, the $\relu(\alpha \cdot (x - (1-\epsilon)))$ filters are also activated and counterbalance the gradients of the other activations.
%The input gradient thus roughly forms a step-function, with inflection points at $\epsilon$ and $1-\epsilon$. For the $\ell_\infty$-norm, most perturbed MNIST pixels remain within the ranges $[0, \epsilon]$ or $[1-\epsilon, 1]$, so first-order adversaries are unaffected by these inflections (i.e., the thresholding filters never transition from active to inactive, or vice versa). For $\ell_1$ or $\ell_2$ attacks however, first-order methods are much more likely to get stuck in local minima. Indeed, the optimal $\ell_1$ or $\ell_2$ attack presumably requires pushing some features outside the $[0, \epsilon] \cup [1-\epsilon, 1]$ range, 

%% file: apx_affine.tex
\ifnips\newpage\fi
\section{Examples of Affine Combinations of Perturbations}
\label{sec:examples}

In Figure~\ref{fig:affine_attacks}, we display examples of $\ell_1$, $\ell_\infty$ and rotation-translation attacks on MNIST and CIFAR10, as well as affine attacks that interpolate between two attack types.

\begin{figure}[h]
	%\vspace{-1em}
		\begin{subfigure}{0.32\textwidth}
			\includegraphics[width=\textwidth]{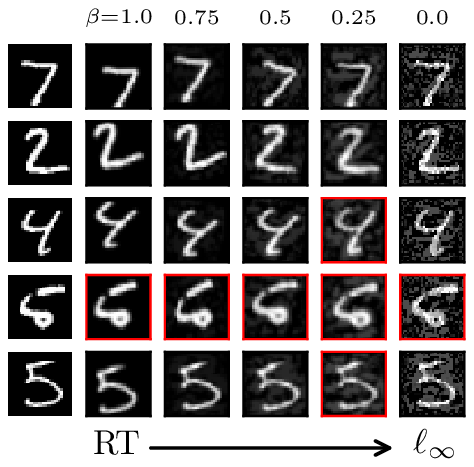}
		\end{subfigure}
		\hfill
		\begin{subfigure}{0.32\textwidth}
			\includegraphics[width=\textwidth]{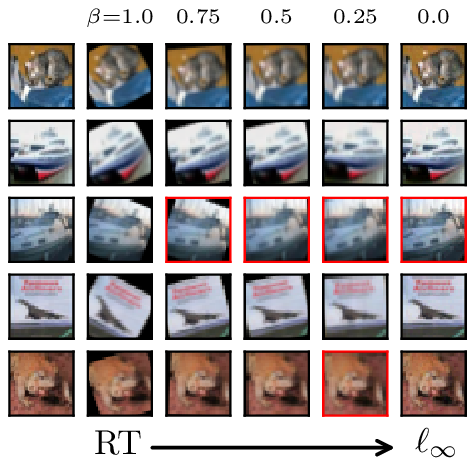}
		\end{subfigure}
		\hfill
		\begin{subfigure}{0.32\textwidth}
			\includegraphics[width=\textwidth]{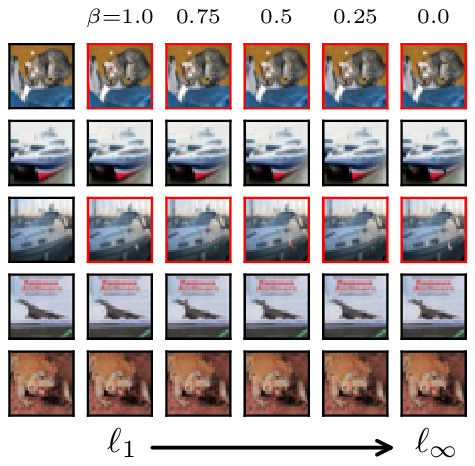}
		\end{subfigure}
		
		\caption{\textbf{Adversarial examples for $\ell_\infty$, $\ell_1$ and rotation-translation (RT) attacks, and affine combinations thereof.} The first column in each subplot shows clean images. The following five images in each row linearly interpolate between two attack types, as described in Section~\ref{ssec:affine}. Images marked in red are mis-classified by a model trained against both types of perturbations. Note that there are examples for which combining a rotation-translation and $\ell_\infty$-attack is stronger than either perturbation type individually.}
		\label{fig:affine_attacks}
\end{figure}

%% file: apx_proofs.tex
%\ifnips\else\newpage\fi
\section{Proof of Theorem~\ref{thm:linf_l1} (Robustness trade-off between \texorpdfstring{$\ell_\infty$}{Linf} and \texorpdfstring{$\ell_1$-}{L1} norms)}
\label{apx:proof_l1_linf}

Our proof follows a similar structure to the proof of Theorem 2.1 in~\cite{tsipras2019robustness}, although the analysis is slightly simplified in our case as we are comparing two perturbation models, an $\ell_\infty$-bounded one and an $\ell_1$-bounded one, that are essentially orthogonal to each other. With a perturbation of size $\epsilon = 2\eta$, the $\ell_\infty$-bounded noise can  ``flip'' the distribution of the features $x_1, \dots, x_d$ to reflect the opposite label, and thus destroy any information that a classifier might extract from those features. On the other side, an $\ell_1$-bounded perturbation with $\epsilon=2$ can flip the distribution of $x_0$. By sacrificing some features, a classifier can thus achieve some robustness to either $\ell_\infty$ \emph{or} $\ell_1$ noise, but never to both simultaneously.

For $y \in \{-1, +1\}$, let $\calG^{y}$ be the distribution over feature $x_0$ conditioned on the value of $y$. Similarly, let $\calH^{y}$ be the conditional distribution over features $x_1, \dots, x_d$.
Consider the following perturbations: $\vec{r}_\infty = [0, -2y\eta, \dots, -2y\eta]$ has small $\ell_\infty$-norm, and $\vec{r}_1 = [-2 x_0, 0, \dots, 0]$ has small $\ell_1$-norm. The $\ell_\infty$ perturbation can change $\calH^{y}$ to $\calH^{-y}$, while the $\ell_1$ perturbation can change $\calG^{y}$ to $\calG^{-y}$.

Let $f(\vec{x})$ be any classifier from $\R^{d+1}$ to $\{-1, +1\}$ and define:
\begin{align*}
	 p_{+-} = \Pr_{\vec{x} \sim (\calG^{+1}, \calH^{-1})}[f(\vec{x}) = +1] \;, \quad\quad
	p_{-+} = \Pr_{\vec{x} \sim (\calG^{-1}, \calH^{+1})}[f(\vec{x}) = +1] \;.
\end{align*}
The accuracy of $f$ against the $\vec{r}_\infty$ perturbation is given by:
\begin{equation*}
\Pr[f(\vec{x}+\vec{r}_\infty) = y] = \Pr[y=+1]\cdot p_{+-} + \Pr[y=-1]\cdot (1-p_{-+}) = \frac12 \cdot (1+p_{+-}-p_{-+}) \;.
\end{equation*}
Similarly, the accuracy of $f$ against the $\vec{r}_1$ perturbation is:
\begin{equation*}
\Pr[f(\vec{x}+\vec{r}_1) = y] = \Pr[y=+1]\cdot p_{-+} + \Pr[y=-1]\cdot (1-p_{+-}) = \frac12 \cdot (1+p_{-+}-p_{+-}) \;.
\end{equation*}
Combining these, we get $\Pr[f(\vec{x}+\vec{r}_\infty) = y] + \Pr[f(\vec{x}+\vec{r}_1) = y]  = 1$. 

As $\vec{r}_\infty$ and $\vec{r}_1$ are two specific $\ell_\infty$- and $\ell_1$-bounded perturbations, the above is an upper-bound on the accuracy that $f$ achieves against worst-case perturbation within the prescribed noise models, which concludes the proof.
 \\
\qed

\section{Proof of Theorem~\ref{thm:linf_RT} (Robustness trade-off between \texorpdfstring{$\ell_\infty$}{Linf} and spatial perturbations)}
\label{apx:proof_RT}

The proof of this theorem follows a similar blueprint to the proof of Theorem~\ref{thm:linf_l1}. Recall that an $\ell_\infty$ perturbation with $\epsilon=2\eta$ can flip the distribution of the features $x_1, \dots, x_n$ to reflect an opposite label $y$. The tricky part of the proof is to show that a small rotation or translation can flip the distribution of $x_0$ to the opposite label, without affecting the marginal distribution of the other features too much.

Recall that we model rotations and translations as picking a permutation $\pi$ from some fixed set $\Pi$ of permutations over the indices in $\vec{x}$, with the constraint that feature $x_0$ be moved to at most $N$ different positions for all $\pi \in \Pi$.

We again define $\calG^y$ as the distribution of $x_0$ conditioned on $y$, and $\calH^y$ for the distribution of $x_1, \dots, x_d$.
We know that a small $\ell_\infty$-perturbation can transform $\calH^y$ into $\calH^{-y}$. Our goal is to show that a rotation-translation adversary can change $(\calG^y, \calH^y)$ into a distribution that is very close to $(\calG^{-y}, \calH^y)$. The result of the theorem then follows by arguing that no binary classifier $f$ can distinguish, with high accuracy, between $\ell_\infty$-perturbed examples with label $y$ and rotated examples with label $-y$ (and vice versa).

We first describe our proof idea at a high level. We define an intermediate ``hybrid'' distribution $\calZ^y$ where all $d+1$ features are i.i.d $N(y\eta, 1)$ (that is, $x_0$ now has the same distribution as the other weakly-correlated features). The main step in the proof is to show that for samples from either $(\calG^y, \calH^y)$ or $(\calG^{-y}, \calH^y)$, a random rotation-translation yields a distribution that is very close (in total variation) to $\calZ^y$. From this, we then show that there exists an adversary that applies two rotations or translations in a row, to first transform samples from $(\calG^y, \calH^y)$ into samples close to $\calZ^y$, and then transform those samples into ones that are close to $(\calG^{-y}, \calH^y)$.

We will need a standard version of the Berry-Esseen theorem, stated hereafter for completeness.
\begin{theorem}[Berry-Esseen~\cite{berry1941accuracy}]
	Let $X_1, \dots, X_n$ be independent random variables with $\Exp[X_i] = \mu_i$, $\Exp[X_i^2] = \sigma_i^2 > 0$, and $\Exp[|X_i|^3] = \rho_i < \infty$, where the $\mu_i, \sigma_i$ and $\rho_i$ are constants independent of $n$. Let $S_{n}=X_{1}+\cdots +X_{n}$, with $F_n(x)$ the CDF of $S_n$ and $\Phi(x)$ the CDF of the standard normal distribution. Then,
	\[
	\sup_{x \in \R} \left|F_n(x) - \Phi\left(\frac{x - \Exp[S_n]}{\sqrt{\Var{[S_n]}}}\right)\right| = O(1/\sqrt{n}) \;.
	\]
\end{theorem}
For distributions $\calP, \calQ$, let $\Delta_{\text{TV}}(\calP, \calQ)$ denote their total-variation distance.  
The below lemma is the main technical result we need, and bounds the total variation between a multivariate Gaussian $\calP$ and a special mixture of multivariate Gaussians $\calQ$.
\begin{lemma}
	\label{lemma:mixture}
	For $k>1$, let $\calP$ be a $k$-dimensional Gaussians with mean $\vec{\mu}_P = [\lambda_P, \dots, \lambda_P]$ and identity covariance. For all $i \in [k]$, let $\calQ_i$ be a multivariate Gaussian with mean $\vec{\mu}_i$ and diagonal covariance $\vec{\Sigma}_i$ where
	$(\vec{\mu}_i)_j = \begin{cases}
	\lambda_Q & \text{if } i = j \\
	\lambda_P & \text{otherwise}
	\end{cases}$ and 
	$(\vec{\Sigma}_i)_{(j, j)} = \begin{cases}
	\sigma_Q^2 & \text{if } i = j \\
	1 & \text{otherwise}
	\end{cases}$.\\
	Define $\calQ$ as a mixture distribution of the $\calQ_1, \dots, \calQ_{k}$ with probabilities $1/k$.
	Assuming that $\lambda_P, \lambda_Q, \sigma_Q$ are constants independent of $k$, we have $\Delta_{\text{TV}}(\calP, \calQ) = O(1/\sqrt{k})$.
\end{lemma}
\begin{proof}\footnote{We thank Iosif Pinelis for his help with this proof (\url{https://mathoverflow.net/questions/325409/}).}
	Let $p(\vec{x})$ and $q(\vec{x})$ denote, respectively, the pdfs of $\calP$ and $\calQ$.
	Note that $q(\vec{x}) = \sum_{i=1}^{k} \frac1k q_i(\vec{x})$, where $q_i(\vec{x})$ is the pdf of $\calQ_i$.
	We first compute:
	\begin{align*}
	q(\vec{x}) &= \sum_{i=1}^k \frac1k  \frac{1}{\sqrt{(2\pi)^{k} \cdot |\vec{\Sigma}_{i}|}} \cdot e^{-\frac12 (\vec{x}-\vec{\mu}_{i})^T \vec{\Sigma}_{i}^{-1}(\vec{x}-\vec{\mu}_{i})} \\
	& = \frac{e^{-\frac12 (\vec{x}-\vec{\mu}_P)^T (\vec{x}-\vec{\mu}_{P})}}{\sqrt{(2\pi)^k}} \cdot \frac{1}{k \cdot \sigma_Q^2} \cdot \sum_{i=1}^k e^{-\frac12 t(x_i)} 
	\\&= p(\vec{x}) \cdot \frac{1}{k \cdot \sigma_Q^2} \cdot \sum_{i=1}^k e^{-\frac12 t(x_i)} \;,
	\end{align*}
	where
	\begin{equation}
	t(x_i) \coloneqq (\sigma_Q^{-2}-1) x_i^2 - (2\lambda_Q\sigma_Q^{-2}-2\lambda_P) x_i + (\lambda_Q^2\sigma_Q^{-2}-\lambda_P^2) \;.
	\end{equation}
	Thus we have that
	\begin{equation*}
	q(\vec{x}) < p(\vec{x}) \quad\Longleftrightarrow\quad \frac{1}{k \cdot \sigma_Q^2} \cdot \sum_{i=1}^k e^{-\frac12 t(x_i)} < 1 \;.
	\end{equation*}
	The total-variation distance between $\calP$ and $\calQ$ is then
	$\Delta_{\text{TV}}(\calP, \calQ) = p_1 - p_2$, where
	\begin{gather}
	\label{eq:p1p2}
	p_1 \coloneqq \Pr\left[S_k < k \cdot \sigma_Q^2\right]\;, \quad
	p_2 \coloneqq \Pr\left[T_k < k \cdot \sigma_Q^2\right] \;, \\
	S_k \coloneqq \sum_{i=1}^k U_i\;, \quad T_k \coloneqq S_{k-1} + V_k \;, \quad U_i \coloneqq e^{-\frac12 t(Z_i)}\;, \quad V_n \coloneqq e^{-\frac12 t(W_n)} \;, \nonumber
	\end{gather}
	and the $Z_i\sim\calN(\lambda_P, 1)$, $W_n \sim \calN(\lambda_Q, \sigma_Q^2)$ and all the $Z_i$ and $W_n$ are mutually independent.
	
	It is easy to verify that $\Exp[U_i] = \sigma_Q^2$, $\Var[U_i] = O(1)$, $\Exp[U_i^3] = O(1)$, $\Exp[W_n] = O(1)$, $\Var[W_n] = O(1), \Exp[W_n^3] = O(1)$. Then, applying the Berry-Esseen theorem, we get:
	\begin{align*}
	p_1 &= \Pr\left[S_k < k \cdot \sigma_Q^2\right] = \Phi\left(0\right) + O\left(\frac{1}{\sqrt{k}}\right) = \frac12 + O\left(\frac{1}{\sqrt{k}}\right)\;, \\
	p_2 &= \Pr\left[T_k < k \cdot \sigma_Q^2\right] = \Phi\left(\frac{k \cdot \sigma_Q^2 - \Exp[T_k]}{\sqrt{\Var[T_k]}}\right) + O\left(\frac{1}{\sqrt{k}}\right) = \Phi \left(O\left(\frac{1}{\sqrt{k}}\right)\right) + O\left(\frac{1}{\sqrt{k}}\right) \\
	&= \frac12 + O\left(\frac{1}{\sqrt{k}}\right)\;.
	\end{align*}
	And thus,
	\begin{align}
	\label{eq:tv}
	\Delta_{\text{TV}}(\calP, \calQ) = p_1 - p_2 = O({1}/{\sqrt{k}})\;.
	\end{align}
\end{proof}

We now define a rotation-translation adversary $\calA$ with a budget of $N$. It samples a random permutation from the set $\Pi$ of permutations that switch position $0$ with a position in $[0, N-1]$ and leave all other positions fixed (note that $|\Pi| = N$). Let $\calA(\calG^y, \calH^y)$ denote the distribution resulting from applying $\calA$ to $(\calG^y, \calH^y)$ and define $\calA(\calG^{-y}, \calH^y)$ similarly.
Recall that $\calZ^y$ is a hybrid distribution which has all features distributed as $\calN(y\eta, 1)$.
\begin{claim}
	$\Delta_{\text{TV}}\left(\calA(\calG^y, \calH^y), \calZ^y\right) = O(1/\sqrt{N})$ and $\Delta_{\text{TV}}\left(\calA(\calG^{-y}, \calH^y), \calZ^y\right) = O(1/\sqrt{N})$
\end{claim}
\begin{proof}
	For the first $N$ features, samples output by $\calA$ follow exactly the distribution $\calQ$ from Lemma~\eqref{lemma:mixture}, for $k=N$ and $\lambda_P=y\cdot\eta, \lambda_Q=y, \sigma_Q^2=\alpha^{-2}$. Note that in this case, the distribution $\calP$ has each feature distributed as in $\calZ^y$. Thus, Lemma~\eqref{lemma:mixture} tells us that the distribution of the first $N$ features is the same as in $\calZ^y$, up to a total-variation distance of $O(1/\sqrt{N})$. As features $x_{N} \dots, x_d$ are unaffected by $\calA$ and thus remain distributed as in $\calZ^y$, we conclude that the total-variation distance between $\calA$'s outputs and $\calZ^y$ is $O(1/\sqrt{N})$.
	
	The proof for $\calA(\calG^{-y}, \calH^y)$ is similar, except that we apply Lemma~\eqref{lemma:mixture} with $\lambda_Q=-y$. 
\end{proof}

Let $\tilde{\calZ}^y$ be the true distribution $\calA(\calG^{-y}, \calH^y)$, which we have shown to be close to $\calZ^y$. Consider the following ``inverse'' adversary $\calA^{-1}$. This adversary samples $\vec{z} \sim \tilde{\calZ}^y$ and returns $\pi^{-1}(\vec{z})$, for $\pi \in \Pi$, with probability
\[
\frac{1}{|\Pi|}\cdot \frac{f_{(\calG^{-y}, \calH^{y})}(\pi^{-1}(\vec{z}))}{f_{\tilde{\calZ}^y}(\vec{z})}\;, 
\]
where $f_{(\calG^{-y}, \calH^{y})}$ and $f_{\tilde{\calZ}^y}$ are the probability density functions for $(\calG^{-y}, \calH^y)$ and for $\tilde{\calZ}^y$. 
\begin{claim}
	$\calA^{-1}$ is a RT adversary with budget $N$ that transforms $\tilde{\calZ}^y$ into $(\calG^{-y}, \calH^y)$.
\end{claim}
\begin{proof}
	Note that $\calA^{-1}$ always applies the inverse of a perturbation in $\Pi$. So feature $x_0$ gets sent to at most $N$ positions when perturbed by $\calA^{-1}$.
	
	Let $Z$ be a random variable distributed as $\tilde{\calZ}^y$ and
	let $h$ be the density function of the distribution obtained by applying $\calA^{-1}$ to $Z$.  We compute:
	\begin{align*}
	h(\vec{x}) &= \sum_{\pi \in \Pi} f_{\tilde{\calZ}^{y}}(\pi(\vec{x})) \cdot \Pr[\calA^{-1} \text{ picks permutation } \pi \mid Z=\pi(\vec{x})] \\
	&= \sum_{\pi \in \Pi} f_{\tilde{\calZ}^{y}}(\pi(\vec{x})) \cdot \frac{1}{|\Pi|}\cdot \frac{f_{(\calG^{-y}, \calH^{y})}(\pi(\pi^{-1}(\vec{x})))}{f_{\tilde{\calZ}^{y}}(\pi(\vec{x}))} = \sum_{\pi \in \Pi} \frac{1}{|\Pi|} \cdot f_{(\calG^{-y}, \calH^{y})}(\vec{x}) \\
	&= f_{(\calG^{-y}, \calH^y)}(\vec{x}) \;,
	\end{align*}
	so applying $\calA^{-1}$ to $\tilde{\calZ}^y$ does yield the distribution $(\calG^{-y}, \calH^y)$.
\end{proof}
We can now finally define our main rotation-translation adversary, $\calA^*$. The adversary first applies $\calA$ to samples from $(\calG^y, \calH^y)$, and then applies $\calA^{-1}$ to the resulting samples from $\tilde{\calZ}^y$.
\begin{claim}
	\label{lemma:tv}
	The adversary $\calA^*$ is a rotation-translation adversary with budget $N$. Moreover,
	%\ifnips
	%$\Delta_{\text{TV}}\left(\calA^*(\calG^y, \calH^y), (\calG^{-y}, \calH^y)\right) = O(1/\sqrt{N})$.
	%\else
	\[\Delta_{\text{TV}}\left(\calA^*(\calG^y, \calH^y), (\calG^{-y}, \calH^y)\right) = O(1/\sqrt{N})\;. \]
	%\fi
\end{claim}
\begin{proof}
	The adversary $\calA^*$ first switches $x_0$ with some random position in $[0, N-1]$ by applying $\calA$. Then, $\calA^{-1}$ either switches $x_0$ back into its original position or leaves it untouched. Thus, $\calA^*$ always moves $x_0$ into one of $N$ positions.
	The total-variation bound follows by the triangular inequality:
	\begin{align*}
	&\Delta_{\text{TV}}\left(\calA^*(\calG^y, \calH^y), (\calG^{-y}, \calH^y)\right) \\
	&\quad\quad= \Delta_{\text{TV}}\left(\calA^{-1}(\calA(\calG^y, \calH^y)), (\calG^{-y}, \calH^y)\right) \\
	&\quad\quad\leq \Delta_{\text{TV}}\left(\calA^{-1}(\calZ^y), (\calG^{-y}, \calH^y)\right) + \Delta_{\text{TV}}\left(\calZ^y, \calA(\calG^y, \calH^y)\right) \\
	&\quad\quad\leq \underbrace{\Delta_{\text{TV}}\left(\calA^{-1}(\tilde{\calZ}^y), (\calG^{-y}, \calH^y)\right)}_{0} + \underbrace{\Delta_{\text{TV}}\left(\tilde{\calZ^y}, (\calG^{-y}, \calH^y)\right)}_{O(1/\sqrt{N})}  + \underbrace{\Delta_{\text{TV}}\left(\calZ^y, \calA(\calG^y, \calH^y)\right)}_{O(1/\sqrt{N})} \\
	&\quad\quad=O(1/\sqrt{N}) \;.
	\end{align*}
\end{proof}

To conclude the proof, we define:
\begin{align*}
p_{+-} = \Pr_{\vec{x} \sim (\calG^{+1}, \calH^{-1})}[f(\vec{x}) = +1] \;, \quad\quad
p_{-+} = \Pr_{\vec{x} \sim (\calG^{-1}, \calH^{+1})}[f(\vec{x}) = +1] \;,
\\
\tilde{p}_{-+} = \Pr_{\vec{x} \sim \calA^*(\calG^{+1}, \calH^{+1})}[f(\vec{x}) = +1] \;, \quad\quad
\tilde{p}_{+-} = \Pr_{\vec{x} \sim (\calG^{-1}, \calH^{-1})}[f(\vec{x}) = +1] \;.
\end{align*}
Then,
\begin{align*}
\Pr[f(\vec{x}+\vec{r}_\infty) = y]  + \Pr[f(A^*(\vec{x})) = y] &= \frac12 p_{+-} + \frac12 (1-p_{-+}) + \frac12\tilde{p}_{-+} + \frac12 (1-\tilde{p}_{+-})  \\
&= 1 + \frac12 \left(p_{+-} - \tilde{p}_{+-}\right) + \frac12 \left(p_{-+} - \tilde{p}_{-+}\right) \\
& \leq 1 - O(1/\sqrt{N}) \;.
\end{align*}
\qed

\subsection{Numerical Estimates for the Robustness Trade-off in Theorem~\ref{thm:linf_RT}}
\label{apx:remark_RT}

While the robustness trade-off we proved in Theorem~\ref{thm:linf_RT} is asymptotic in $N$ (the budget of the RT adversary), we can provide tight numerical estimates for this trade-off for concrete parameter settings:
\begin{myremark}
	\label{remark:RT}
	Let $d\geq 200$, $\alpha=2$ and $N=49$ (e.g., translations by $\pm3$ pixels). Then, there exists a classifier with $\calR_{\text{adv}}(f; S_{\infty}) <10\%$, as well as a (distinct) classifier with $\calR_{\text{adv}}(f; S_{\text{RT}})<10\%$. Yet, any single classifier  satisfies $\calR_{\text{adv}}^{\text{avg}}(f; S_{\infty}, S_{\text{RT}}) \gtrapprox 0.425$.
\end{myremark}

We first show the existence of classifiers with $\calR_{\text{adv}} < 10\%$ for the given $\ell_\infty$ and RT attacks.

Let $f(\vec{x}) = \sign(x_0)$ and let $\vec{r} = [-y\epsilon, 0, \dots, 0]$ be the worst-case perturbation with $\norm{\vec{r}} \leq \epsilon$. Recall that $\epsilon = 2\eta = 4/\sqrt{d}$. We have
\[
\Pr[f(\vec{x}+\vec{r}) \neq y] = \Pr[\calN(1, 1/4) - 4/\sqrt{d} < 0] \leq \Pr[\calN(1-4/\sqrt{200}, 1/4)< 0] \leq 8\% \;.
\]
Thus, $f$ achieves $\calR_{\text{adv}} < 10\%$ against the $\ell_\infty$-perturbations.

Let $g(\vec{x}) = \sign(\sum_{i=N}^d x_i)$ be a classifier that ignores all feature positions that a RT adversary $\calA$ may affect. We have
\begin{align*}
\Pr[g(\calA(\vec{x})) \neq y] &= \Pr[g(\vec{x}) \neq y] = \Pr\left[\calN\left((d-N+1)\cdot \eta, d-N+1\right) < 0\right] \\
&\leq \Pr[\calN(2\sqrt{d-48}/\sqrt{d}, 1) < 0] \leq 5\% \;.
\end{align*}
Thus, $g$ achieves $\calR_{\text{adv}} < 10\%$ against RT perturbations.

We upper-bound the adversarial risk that any classifier must incur against both attacks by numerically estimating the total-variation distance between the distributions induced by the RT and $\ell_\infty$ adversaries for inputs of opposing labels $y$.
Specifically, we generate $100{,}000$ samples from the distributions $\calG^{+1}, \calG^{-1}$ and $\calH^{+1}$ as defined in the proof of Theorem~\ref{thm:linf_RT}, and obtain an estimate of the total-variation distance in Lemma~\eqref{lemma:tv}. For this, we numerically estimate $p_1$ and $p_2$ as defined in Equation~\eqref{eq:p1p2}.

\section{Proof of Claim~\ref{thm:linear_affine} (Affine combinations of \texorpdfstring{$\ell_p$-}{Lp} perturbations do not affect linear models)}
\label{apx:proof_affine_linear}

Let 
\[
\max_{\vec{r} \in S_{\text{U}}} \vec{w}^T \vec{r} = v_{\text{max}}, \quad \text{and} \quad \min_{\vec{r} \in S_{\text{U}}} \vec{w}^T \vec{r} = v_{\text{min}} \;.
\] 
Let $S_{\text{U}} \coloneqq S_p \cup S_q$. Note that any $\vec{r} \in S_{\text{affine}}$ is of the form $\beta \vec{r}_1 + (1-\beta) \vec{r}_2$ for $\beta \in [0,1]$. Moreover, we have $\vec{r}_1 \in S_{p} \subset S_{\text{U}}$ and $\vec{r}_2 \in S_{q} \subset S_{\text{U}}$. Thus, 
\[
\max_{\vec{r} \in S_{\text{affine}}} \vec{w}^T \vec{r} = v_{\text{max}}, \quad \text{and} \quad \min_{\vec{r} \in S_{\text{affine}}} \vec{w}^T \vec{r} = v_{\text{min}} \;.
\]
Let $h(\vec{x}) = \vec{w}^T \vec{x} + b$, so that $f(\vec{x}) = \sign(h(\vec{x}))$. Then, we get
\begin{align*}
\Pr_{\calD}\left[\exists \vec{r} \in S_\text{affine}: f(\vec{x}+\vec{r}) \neq y\right] 
& = \frac12 \Pr_{\calD}\left[\exists \vec{r} \in  S_\text{affine}: \vec{w}^T \vec{r} < -h(\vec{x}) \mid y=+1\right] \\
& \quad\quad + \frac12 \Pr_{\calD}\left[\exists \vec{r} \in  S_\text{affine}: \vec{w}^T \vec{r} > h(\vec{x}) \mid y=-1\right] \\
& = \frac12 \Pr_{\calD}\left[v_{\text{min}} < -h(\vec{x}) \mid y=+1\right] + \frac12 \Pr_{\calD}\left[v_{\text{max}} > h(\vec{x}) \mid y=-1\right] \\
& = \frac12 \Pr_{\calD}\left[\exists \vec{r} \in  S_\text{U}: \vec{w}^T \vec{r} < -h(\vec{x}) \mid y=+1\right] \\
&\quad\quad + \frac12 \Pr_{\calD}\left[\exists \vec{r} \in  S_\text{U}: \vec{w}^T \vec{r} > h(\vec{x}) \mid y=-1\right] \\
&=\Pr_{\calD}\left[\exists \vec{r} \in S_\text{U}: f(\vec{x}+\vec{r}) \neq y\right] \;.
\end{align*}
\qed

\section{Affine combinations of \texorpdfstring{$\ell_p$-}{Lp} perturbations can affect non-linear models}
\label{apx:affine_LP_nonlinear}

In Section~\ref{ssec:affine}, we showed that for linear models, robustness to a union of  $\ell_p$-perturbations implies robustness to an affine adversary that interpolates between perturbation types. We show that this need not be the case when the model is non-linear.
In particular, we can show that for the distribution $\calD$ introduced in Section~\ref{sec:theory}, non-linearity is necessary to achieve robustness to a union of
$\ell_\infty$ and $\ell_1$-perturbations (with different parameter settings than for Theorem~\ref{thm:linf_l1}), but that at the same time, robustness to affine combinations of these perturbations is unattainable by any model.
\begin{theorem}%[Affine combinations of $\ell_p$-perturbations can affect non-linear models]
	\label{thm:nonlinear_affine}
	Consider the distribution $\calD$ with $d \geq 200$, $\alpha = 2$ and $p_0 = 1-\Phi(-2)$. Let $S_\infty$ be the set of $\ell_\infty$-bounded perturbation with $\epsilon = (3/2)\eta = {3}/{\sqrt{d}}$ and let $S_1$ be the set of $\ell_1$-bounded perturbations with $\epsilon = 3$. Define $S_\text{affine}$ as in Section~\ref{ssec:affine}. Then, there exists a non-linear classifier $g$ that achieves $\calR_{\text{adv}}^{\text{max}}(g; S_\infty, S_1) \leq 35\%$. Yet, for all classifiers $f$ we have $\calR_{\text{adv}}(f; S_\text{affine}) \geq 50\%$.
\end{theorem}
\begin{proof}
We first prove that no classifier can achieve accuracy above $50\%$ (which is achieved by the constant classifier) against $S_\text{affine}$. The proof is very similar to the one of Theorem~\ref{thm:linf_l1}.

Let $\beta = 2/3$, so the affine attacker gets to compose an $\ell_\infty$-budget of ${2}/{\sqrt{d}}$ and an $\ell_1$-budget of $1$. Specifically, for a point $(\vec{x}, y) \sim \calD$, the affine adversary will apply the perturbation
\[
\vec{r} = [-x_0, -y\frac{2}{\sqrt{d}}, \dots, -y\frac{2}{\sqrt{d}}] = [-x_0, -y\eta, \dots, -y\eta]\;.
\]
Let $\calG^{0,0}$ be the following distribution:
\[
y \stackrel{\emph{u.a.r}}{\sim} \{-1, +1\}, \quad 
x_0 = 0, \quad
x_1, \dots, x_d \stackrel{\emph{i.i.d}}{\sim} \calN(0, 1) \;.
\]
Note that in $\calG^{0,0}$, $\vec{x}$ is independent of $y$ so no classifier can achieve more than $50\%$ accuracy on $\calG^{0,0}$. Yet, note that the affine adversary's perturbation $\vec{r}$ transforms any $(\vec{x}, y) \sim \calD$ into $(\vec{x}, y) \sim \calG^{0, 0}$.

We now show that there exists a classifier that achieves non-trivial robustness against the set of perturbations $S_\infty \cup S_1$, i.e., the union of $\ell_\infty$-noise with $\epsilon={3}/{\sqrt{d}}$ and $\ell_1$-noise with $\epsilon=3$. Note that by Claim~\ref{thm:linear_affine}, this classifier must be \emph{non-linear}. We define
\[
f(\vec{x}) = \sign\left(3\cdot \sign(x_0) + \sum_{i=1}^d \frac{2}{\sqrt{d}} \cdot x_i\right) \;.
\]
The reader might notice that $f(\vec{x})$ closely resembles the \emph{Bayes optimal classifier} for $\calD$ (which would be a linear classifier). The non-linearity in $f$ comes from the $\sign$ function applied to $x_0$. Intuitively, this limits the damage caused by the $\ell_1$-noise, as $\sign(x_0)$ cannot change by more than $\pm2$ under any perturbation of $x_0$. This forces the $\ell_1$ perturbation budget to be ``wasted'' on the other features $x_1, \dots, x_d$, which are very robust to $\ell_1$ attacks.

As a warm-up, we compute the classifier's natural accuracy on $\calD$. 
For $(\vec{x}, y) \sim \calD$, let $X = y \cdot \sum_{i=1}^d \frac{2}{\sqrt{d}} \cdot x_i$ be a random variable. Recall that $\eta=2/\sqrt{d}$. Note that $X$ is distributed as
\[
y \cdot \sum_{i=1}^d \frac{2}{\sqrt{d}} \cdot\calN(y\eta, 1) = \sum_{i=1}^d \frac{2}{\sqrt{d}}\cdot \calN\left(\frac{2}{\sqrt{d}}, 1\right) = \sum_{i=1}^d \calN\left(\frac{4}{d}, \frac{4}{d}\right) = \calN(4, 4) \;.
\]
Recall that $x_0 = y$ with probability $p_0 = 1-\Phi(-2) \approx 0.977$. We get:
\begin{align*}
\Pr_\calD[f(\vec{x}) = y] &= \Pr_\calD \left[y \cdot \left(3 \cdot \sign(x_0) + \sum_{i=1}^d \frac{2}{\sqrt{d}} \cdot x_i \right) > 0\right] \\
& = \Pr_\calD[x_0 = y] \cdot \Pr_\calD \left[3 \cdot y \cdot \sign(x_0) + X > 0 \mid x_0 = y\right] \\
& \quad+ \Pr_\calD[x_0 \neq y] \cdot \Pr_\calD \left[3 \cdot y \cdot \sign(x_0) + X > 0 \mid x_0 \neq y\right] \\
&= p\cdot \Pr \left[3 + \calN(4, 4) > 0\right] + (1-p)\cdot \Pr \left[-3 + \calN(4, 4) > 0\right] \approx 99\% \;.
\end{align*}

We now consider an adversary that picks either an $\ell_\infty$-perturbation with $\epsilon={3}/{\sqrt{d}}$ or an $\ell_1$-perturbation with $\epsilon=3$. 
It will suffice to consider the case where $x_0 = y$. Note that the $\ell_\infty$ classifier cannot meaningfully perturb $x_0$, and the best perturbation is always $\vec{r}_\infty = [0, -y{3}/{\sqrt{d}}, \dots, -y{3}/{\sqrt{d}}]$. Moreover, the best $\ell_1$-bounded perturbation is $\vec{r}_1 = [-2y, -y, 0, \dots, 0]$. We have $f(\vec{x}+\vec{r}_\infty) = \sign(y \cdot (3 + X -6))$ and $f(\vec{x}+\vec{r}_1) = \sign(y \cdot (-3 + X - {2}/{\sqrt{d}}))$.
We now lower-bound the classifier's accuracy under the union $S_{\text{U}} \coloneqq S_\infty \cup S_1$ of these two perturbation models:
\begin{align*}
\Pr_\calD[f(\vec{x} + \vec{r}) = y, \forall \vec{r} \in S_\text{U}] 
&\geq \Pr_{\calD}[x_0 = y] \cdot \Pr_\calD[f(\vec{x} + \vec{r}) = y, \forall \vec{r} \in S_\text{U} \mid x_0 = y]  \\
& \geq p \cdot \Pr_\calD\left[(3 + X -6 > 0) \wedge (-3 + X - {2}/{\sqrt{d}}) > 0)\right] \\
& = p \cdot \Pr\left[\calN(4, 4) > 3 + {2}/{\sqrt{d}} \right]
\geq 65\% \quad (\text{for } d \geq 200) \;.
\end{align*}
\end{proof}

\section{Proof of Theorem~\ref{thm:affine_RT} (Affine combinations of \texorpdfstring{$\ell_\infty$-}{Linf} and spatial perturbations can affect linear models)}
\label{apx:proof_affine_RT}

Note that our definition of affine perturbation allows for a different weighting parameter $\beta$ to be chosen for each input. Thus, the adversary that selects perturbations from $S_{\text{affine}}$ is at least as powerful as the one that selects perturbations from $S_\infty \cup S_{\text{RT}}$.
All we need to show to complete the proof is that there exists some input $\vec{x}$ that the affine adversary can perturb, while the adversary limited to the union of spatial and $\ell_\infty$ perturbations cannot.

Without loss of generality, assume that the RT adversary picks a permutation that switches $x_0$ with a position in $[0, N-1]$, and leaves all other indices untouched. The main idea is that for any input $\vec{x}$ where the RT adversary moves $x_0$ to position $j < N-1$, the RT adversary with budget $N$ is no more powerful than one with budget $j+1$. The affine adversary can thus limit its rotation-translation budget and use the remaining budget on an extra $\ell_\infty$ perturbation.

We now construct an input $\vec{x}$ such that: (1) $\vec{x}$ cannot be successfully attacked by an RT adversary (with budget $N$) or by an $\ell_\infty$-adversary (with budget $\epsilon$); (2) $\vec{x}$ can be attacked by an affine adversary.

Without loss of generality, assume that $w_1 = \min\{w_1, \dots, w_{N-1}\}$, i.e., among all the features that $x_0$ can be switched with, $x_1$ has the smallest weight.
Let $y=+1$, and let $x_1, \dots, x_{N-1}$ be chosen such that $\argmin\{x_1, \dots, x_{N-1}\} = 1$. We set
\[
x_0 \coloneqq \frac{\epsilon \cdot \norm{\vec{w}}_1}{w_0 - w_1} + x_1 \;.
\]
Moreover, set $x_{N}, \dots, x_d$ such that
\[
\vec{w}^T\vec{x}+b = 1.1 \cdot \epsilon \cdot \norm{\vec{w}}_1 \;.
\]
Note that constructing such an $\vec{x}$ is always possible as we assumed $w_0 > w_i > 0$ for all $1 \leq i \leq d$. 

We now have an input $(\vec{x}, y)$ that has non-zero support under $\calD$. Let $\vec{r}$ be a perturbation with $\norm{\vec{r}}_\infty \leq \epsilon$. We have:
\[
\vec{w}^T(\vec{x}+\vec{r})+b \geq  \vec{w}^T\vec{x}+b - \epsilon \cdot \norm{\vec{w}}_1 = 0.1 \cdot \epsilon \cdot \norm{\vec{w}}_1 > 0 \;,
\]
so $f(\vec{w}^T(\vec{x}+\vec{r})+b) = y$, i.e., $\vec{x}$ cannot be attacked by any $\epsilon$-bounded $\ell_\infty$-perturbation.

Define $\hat{\vec{x}}_i$ as the input $\vec{x}$ with features $x_0$ and $x_i$ switched, for some $0 \leq i < N$. Then,
\begin{align*}
\vec{w}^T\hat{\vec{x}}_i+b &= \vec{w}^T\vec{x}+b - (w_0 - w_i) \cdot (x_0 - x_i) \\
&\geq  \vec{w}^T\vec{x}+b - (w_0 - w_1) \cdot (x_0 - x_1) \\
&= \vec{w}^T\vec{x}+b - \epsilon \cdot \norm{\vec{w}}_1 = 0.1 \cdot \epsilon \cdot \norm{\vec{w}}_1 > 0 \;.
\end{align*}
Thus, the RT adversary cannot change the sign of $f(\vec{x})$ either. This means that an adversary that chooses from $S_\infty \cup S_{\text{RT}}$ cannot successfully perturb $\vec{x}$.

Now, consider the affine adversary, with $\beta=2/N$ that first applies an RT perturbation with budget $\frac{2}{N} \cdot N = 2$ (i.e., the adversary can only flip $x_0$ with $x_1$), followed by an $\ell_\infty$-perturbation with budget $(1-\frac{2}{N}) \cdot \epsilon$. Specifically, the adversary flips $x_0$ and $x_1$ and then adds noise $\vec{r} = -(1-\frac{2}{N}) \cdot \epsilon \cdot \sign(\vec{w})$. Let this adversarial example by $\hat{\vec{x}}_{\text{affine}}$. We have
\begin{align*}
\vec{w}^T\hat{\vec{x}}_{\text{affine}}+b &= \vec{w}^T\vec{x}+b -(w_0 - w_1) \cdot (x_0 - x_1) - \left(1-\frac{2}{N}\right) \cdot \epsilon \cdot \norm{\vec{w}}_1 \\
& = 1.1\cdot \epsilon \cdot \norm{\vec{w}}_1 - \epsilon \cdot \norm{\vec{w}}_1 - \left(1-\frac{2}{N}\right) \cdot \epsilon \cdot \norm{\vec{w}}_1\\
& = -\left(0.9-\frac{2}{N}\right) \cdot \epsilon \cdot \norm{\vec{w}}_1\\
& <0 \;.
\end{align*}
Thus, $f(\hat{\vec{x}}_{\text{affine}}) = -1 \neq y$, so the affine adversary is strictly stronger that the adversary that is restricted to RT or $\ell_\infty$ perturbations. \qed

%% file: paper.bbl
\begin{thebibliography}{10}

\bibitem{athalye2018obfuscated}
A.~Athalye, N.~Carlini, and D.~Wagner.
\newblock Obfuscated gradients give a false sense of security: Circumventing
  defenses to adversarial examples.
\newblock In {\em International Conference on Machine Learning (ICML)}, 2018.

\bibitem{berry1941accuracy}
A.~C. Berry.
\newblock The accuracy of the gaussian approximation to the sum of independent
  variates.
\newblock {\em Transactions of the american mathematical society},
  49(1):122--136, 1941.

\bibitem{brendel2018decision}
W.~Brendel, J.~Rauber, and M.~Bethge.
\newblock Decision-based adversarial attacks: Reliable attacks against
  black-box machine learning models.
\newblock In {\em International Conference on Learning Representations}, 2018.

\bibitem{carlini2018prototypical}
N.~Carlini, U.~Erlingsson, and N.~Papernot.
\newblock Prototypical examples in deep learning: Metrics, characteristics, and
  utility.
\newblock 2018.

\bibitem{carlini2016hidden}
N.~Carlini, P.~Mishra, T.~Vaidya, Y.~Zhang, M.~Sherr, C.~Shields, D.~Wagner,
  and W.~Zhou.
\newblock Hidden voice commands.
\newblock In {\em USENIX Security Symposium}, pages 513--530, 2016.

\bibitem{carlini2016towards}
N.~Carlini and D.~Wagner.
\newblock Towards evaluating the robustness of neural networks.
\newblock In {\em IEEE Symposium on Security and Privacy}, 2017.

\bibitem{chen2019boundary}
J.~Chen and M.~I. Jordan.
\newblock Boundary attack++: Query-efficient decision-based adversarial attack.
\newblock {\em arXiv preprint arXiv:1904.02144}, 2019.

\bibitem{chen2018ead}
P.-Y. Chen, Y.~Sharma, H.~Zhang, J.~Yi, and C.-J. Hsieh.
\newblock Ead: elastic-net attacks to deep neural networks via adversarial
  examples.
\newblock In {\em AAAI Conference on Artificial Intelligence}, 2018.

\bibitem{demontis2016security}
A.~Demontis, P.~Russu, B.~Biggio, G.~Fumera, and F.~Roli.
\newblock On security and sparsity of linear classifiers for adversarial
  settings.
\newblock In {\em Joint IAPR International Workshops on Statistical Techniques
  in Pattern Recognition (SPR) and Structural and Syntactic Pattern Recognition
  (SSPR)}, pages 322--332. Springer, 2016.

\bibitem{duchi2008efficient}
J.~Duchi, S.~Shalev-Shwartz, Y.~Singer, and T.~Chandra.
\newblock Efficient projections onto the l1-ball for learning in high
  dimensions.
\newblock In {\em International Conference on Machine Learning (ICML)}, 2008.

\bibitem{engstrom2017rotation}
L.~Engstrom, B.~Tran, D.~Tsipras, L.~Schmidt, and A.~Madry.
\newblock A rotation and a translation suffice: Fooling {CNNs} with simple
  transformations.
\newblock {\em arXiv preprint arXiv:1712.02779}, 2017.

\bibitem{fawzi2018adversarial}
A.~Fawzi, H.~Fawzi, and O.~Fawzi.
\newblock Adversarial vulnerability for any classifier.
\newblock In {\em Advances in Neural Information Processing Systems}, pages
  1186--1195, 2018.

\bibitem{geirhos2018imagenet}
R.~Geirhos, P.~Rubisch, C.~Michaelis, M.~Bethge, F.~A. Wichmann, and
  W.~Brendel.
\newblock {ImageNet}-trained {CNNs} are biased towards texture; increasing
  shape bias improves accuracy and robustness.
\newblock In {\em International Conference on Learning Representations (ICLR)},
  2019.

\bibitem{gilmer2018adversarial}
J.~Gilmer, L.~Metz, F.~Faghri, S.~S. Schoenholz, M.~Raghu, M.~Wattenberg, and
  I.~Goodfellow.
\newblock Adversarial spheres.
\newblock {\em arXiv preprint arXiv:1801.02774}, 2018.

\bibitem{goodfellow2014explaining}
I.~J. Goodfellow, J.~Shlens, and C.~Szegedy.
\newblock Explaining and harnessing adversarial examples.
\newblock In {\em International Conference on Learning Representations (ICLR)},
  2015.

\bibitem{grosse2016adversarial}
K.~Grosse, N.~Papernot, P.~Manoharan, M.~Backes, and P.~McDaniel.
\newblock Adversarial examples for malware detection.
\newblock In {\em European Symposium on Research in Computer Security}, 2017.

\bibitem{hendrycks2018benchmarking}
D.~Hendrycks and T.~Dietterich.
\newblock Benchmarking neural network robustness to common corruptions and
  perturbations.
\newblock In {\em International Conference on Learning Representations (ICLR)},
  2019.

\bibitem{ilyas2019adversarial}
A.~Ilyas, S.~Santurkar, D.~Tsipras, L.~Engstrom, B.~Tran, and A.~Madry.
\newblock Adversarial examples are not bugs, they are features.
\newblock {\em arXiv preprint arXiv:1905.02175}, 2019.

\bibitem{jo2017measuring}
J.~Jo and Y.~Bengio.
\newblock Measuring the tendency of {CNNs} to learn surface statistical
  regularities.
\newblock {\em arXiv preprint arXiv:1711.11561}, 2017.

\bibitem{kang2019transfer}
D.~Kang, Y.~Sun, T.~Brown, D.~Hendrycks, and J.~Steinhardt.
\newblock Transfer of adversarial robustness between perturbation types.
\newblock {\em arXiv preprint arXiv:1905.01034}, 2019.

\bibitem{khoury2019geometry}
M.~Khoury and D.~Hadfield-Menell.
\newblock On the geometry of adversarial examples, 2019.

\bibitem{kurakin2016scale}
A.~Kurakin, I.~Goodfellow, and S.~Bengio.
\newblock Adversarial machine learning at scale.
\newblock In {\em International Conference on Learning Representations (ICLR)},
  2017.

\bibitem{li2018second}
B.~Li, C.~Chen, W.~Wang, and L.~Carin.
\newblock Second-order adversarial attack and certifiable robustness.
\newblock {\em arXiv preprint arXiv:1809.03113}, 2018.

\bibitem{madry2018tutorial}
A.~Madry and Z.~Kolter.
\newblock Adversarial robustness: Theory and practice.
\newblock In {\em Tutorial at NeurIPS 2018}, 2018.

\bibitem{madry2018towards}
A.~Madry, A.~Makelov, L.~Schmidt, D.~Tsipras, and A.~Vladu.
\newblock Towards deep learning models resistant to adversarial attacks.
\newblock In {\em International Conference on Learning Representations (ICLR)},
  2018.

\bibitem{mahloujifar2018curse}
S.~Mahloujifar, D.~I. Diochnos, and M.~Mahmoody.
\newblock The curse of concentration in robust learning: Evasion and poisoning
  attacks from concentration of measure.
\newblock {\em arXiv preprint arXiv:1809.03063}, 2018.

\bibitem{papernot2016practical}
N.~Papernot, P.~McDaniel, I.~Goodfellow, S.~Jha, Z.~B. Celik, and A.~Swami.
\newblock Practical black-box attacks against machine learning.
\newblock In {\em ASIACCS}, pages 506--519. ACM, 2017.

\bibitem{raghunathan2018certified}
A.~Raghunathan, J.~Steinhardt, and P.~Liang.
\newblock Certified defenses against adversarial examples.
\newblock In {\em International Conference on Learning Representations (ICLR)},
  2018.

\bibitem{ribeiro2016should}
M.~T. Ribeiro, S.~Singh, and C.~Guestrin.
\newblock Why should i trust you?: Explaining the predictions of any
  classifier.
\newblock In {\em KDD}. ACM, 2016.

\bibitem{schmidt2018adversarially}
L.~Schmidt, S.~Santurkar, D.~Tsipras, K.~Talwar, and A.~Madry.
\newblock Adversarially robust generalization requires more data.
\newblock In {\em Advances in Neural Information Processing Systems}, pages
  5019--5031, 2018.

\bibitem{schott2018towards}
L.~Schott, J.~Rauber, M.~Bethge, and W.~Brendel.
\newblock Towards the first adversarially robust neural network model on mnist.
\newblock In {\em International Conference on Learning Representations (ICLR)},
  2019.

\bibitem{openreview}
L.~Schott, J.~Rauber, M.~Bethge, and W.~Brendel.
\newblock Towards the first adversarially robust neural network model on mnist
  ({OpenReview} comment on spatial transformations), 2019.

\bibitem{shafahi2019adversarial}
A.~Shafahi, W.~R. Huang, C.~Studer, S.~Feizi, and T.~Goldstein.
\newblock Are adversarial examples inevitable?
\newblock In {\em International Conference on Learning Representations (ICLR)},
  2019.

\bibitem{shafahi2019free}
A.~Shafahi, M.~Najibi, A.~Ghiasi, Z.~Xu, J.~Dickerson, C.~Studer, L.~S. Davis,
  G.~Taylor, and T.~Goldstein.
\newblock Adversarial training for free!
\newblock {\em arXiv preprint arXiv:1904.12843}, 2019.

\bibitem{sharma2017attacking}
Y.~Sharma and P.-Y. Chen.
\newblock Attacking the madry defense model with l1-based adversarial examples.
\newblock {\em arXiv preprint arXiv:1710.10733}, 2017.

\bibitem{stock2018convnets}
P.~Stock and M.~Cisse.
\newblock Convnets and imagenet beyond accuracy: Understanding mistakes and
  uncovering biases.
\newblock In {\em Proceedings of the European Conference on Computer Vision
  (ECCV)}, pages 498--512, 2018.

\bibitem{szegedy2013intriguing}
C.~Szegedy, W.~Zaremba, I.~Sutskever, J.~Bruna, D.~Erhan, I.~Goodfellow, and
  R.~Fergus.
\newblock Intriguing properties of neural networks.
\newblock In {\em International Conference on Learning Representations (ICLR)},
  2014.

\bibitem{fullversion}
F.~Tram{\`e}r and D.~Boneh.
\newblock Adversarial training and robustness for multiple perturbations.
\newblock In {\em Neural Information Processing Systems (NeurIPS) 2019}, 2019.
\newblock arXiv preprint arXiv:1904.13000.

\bibitem{tramer2018adblock}
F.~Tram{\`e}r, P.~Dupr{\'e}, G.~Rusak, G.~Pellegrino, and D.~Boneh.
\newblock Ad-versarial: Perceptual ad-blocking meets adversarial machine
  learning.
\newblock arXiv preprint arXiv:1811:03194, Nov 2018.

\bibitem{tramer2018ensemble}
F.~Tram{\`e}r, A.~Kurakin, N.~Papernot, I.~Goodfellow, D.~Boneh, and
  P.~McDaniel.
\newblock Ensemble adversarial training: Attacks and defenses.
\newblock In {\em International Conference on Learning Representations (ICLR)},
  2018.

\bibitem{tsipras2019robustness}
D.~Tsipras, S.~Santurkar, L.~Engstrom, A.~Turner, and A.~Madry.
\newblock Robustness may be at odds with accuracy.
\newblock In {\em International Conference on Learning Representations (ICLR)},
  2019.

\bibitem{wong2018provable}
E.~Wong and Z.~Kolter.
\newblock Provable defenses against adversarial examples via the convex outer
  adversarial polytope.
\newblock In {\em International Conference on Machine Learning}, pages
  5283--5292, 2018.

\bibitem{xu2009robustness}
H.~Xu, C.~Caramanis, and S.~Mannor.
\newblock Robustness and regularization of support vector machines.
\newblock {\em Journal of Machine Learning Research}, 10(Jul):1485--1510, 2009.

\bibitem{zhang2019you}
D.~Zhang, T.~Zhang, Y.~Lu, Z.~Zhu, and B.~Dong.
\newblock You only propagate once: Painless adversarial training using maximal
  principle.
\newblock {\em arXiv preprint arXiv:1905.00877}, 2019.

\end{thebibliography}
